\documentclass[10pt]{article}

\usepackage[preprint]{tmlr}

\usepackage{algorithm}
\usepackage{algorithmic}

\usepackage{amsmath}
\usepackage{amsthm}
\usepackage{amssymb}
\usepackage{thmtools}
\usepackage{xcolor}
\usepackage{bm}
\usepackage{dsfont}
\usepackage{mathtools}
\usepackage{hyperref}

\usepackage{pgfplots}
\usepackage{tikz}
\usetikzlibrary{arrows,trees}

\pgfplotsset{compat=1.18}

\newcommand{\opt}{^\star}

\newcommand{\Real}{\mathbb{R}}

\newcommand{\argmax}{\mathop{\mathrm{arg\,max}}\limits}
\newcommand{\argmin}{\mathop{\mathrm{arg\,min}}\limits}

\DeclareMathOperator{\cvaro}{CVaR}
\DeclareMathOperator{\ess}{ess}

\newcommand{\cvar}[2]{\cvaro_{#1} \left[#2\right]}

\newcommand{\cvarEnv}[2]{\Xi_{#1}\left(#2\right)}

\newcommand{\ind}[1]{\mathds{1}{\left[{#1}\right]}}  

\newcommand{\states}{\mathcal{S}}
\newcommand{\actions}{\mathcal{A}}

\newcommand{\V}[2]{\bm \upsilon^{#1}_{#2}}

\renewcommand*{\cite}[2][]{\citep[#1]{#2}}

\definecolor{darkspringgreen}{rgb}{0.09, 0.45, 0.27}
\definecolor{darkpastelblue}{rgb}{0.47, 0.62, 0.8}
\definecolor{darkslateblue}{rgb}{0.28, 0.24, 0.55}
\definecolor{azure(colorwheel)}{rgb}{0.0, 0.5, 1.0}
\definecolor{warmorange}{RGB}{217, 95, 2}    
\definecolor{awesome}{rgb}{1.0, 0.13, 0.32}

\newtheorem{remark}{Remark}
\newtheorem{theorem}{Theorem}
\newtheorem{definition}{Definition}

\newtheorem{corollary}{Corollary}

\newtheorem{assumptions}{Assumptions}

\usepackage{lastpage}

\title{On the Fundamental Limitations of Dual Static CVaR\\ Decompositions in Markov Decision Processes}

\author{\name Mathieu Godbout \email mathieu.godbout.3@ulaval.ca \\
       \addr Department of Computer Science\\
       Université Laval, Canada
       \AND
       \name Audrey Durand \email audrey.durand@ift.ulaval.ca\\
       \addr Canada-CIFAR AI Chair\\Department of Computer Science\\
       Université Laval, Canada}

\def\month{MM}  
\def\year{YYYY} 
\def\openreview{\url{https://openreview.net/forum?id=XXXX}} 

\begin{document}

\maketitle

\begin{abstract}%
%
It was recently shown that dynamic programming (DP) methods for finding static CVaR-optimal policies in Markov Decision Processes (MDPs) can fail when based on the dual formulation, yet the root cause of this failure remains unclear. We expand on these findings by shifting focus from policy optimization to the seemingly simpler task of policy evaluation. We show that evaluating the static CVaR of a given policy can be framed as two distinct minimization problems. We introduce a set of ``risk-assignment consistency constraints'' that must be satisfied for their solutions to match 
and we demonstrate that an empty intersection of these constraints is the source of previously observed evaluation errors. Quantifying the evaluation error as the \emph{CVaR evaluation gap}, we demonstrate that the issues observed when optimizing over the dual-based CVaR DP are explained by the returned policy having a non-zero CVaR evaluation gap. Finally, we leverage our proposed risk-assignment constraints perspective to prove that the search for a single, uniformly optimal policy on the dual CVaR decomposition is fundamentally limited, identifying an MDP where no single policy can be optimal across all initial risk levels.

\end{abstract}


%

\section{Introduction}
\label{sec:intro}





The goal of reinforcement learning (RL)~\citep{Sutton2018Reinforcement} is to learn (sequential) decision-making policies such as to maximize some outcome (return) in a given environment, typically modeled as a Markov decision process (MDP).
This is usually approached from the objective of maximizing the \textit{expected} return, which has led to impressive successes in games~\citep{Silver2018General,Vinyals2019Grandmaster} and content recommendation~\citep{Li2010Contextual}.
However, in safety-critical domains like healthcare, autonomous driving, or financial planning, some erroneous actions may lead to disastrous consequences. For instance, in the task of identifying the shortest path to an organ for surgery, paths at high risk of endangering the patient (e.g., as they are too close to an artery, a nerve, or a critical region of the brain), should be avoided~\citep{Baek2018Path}. 
Automation in this context and other safety-critical domains therefore requires safe decision-making policies~\citep{Gottesman2019Guidelines}. This can be achieved by optimizing a \textit{risk-averse} objective instead of simply maximizing the expected return~\citep{Artzner1999Coherent}. 
In particular, conditional value-at-risk (CVaR), which is considered a gold standard risk measure in banking regulations \citep{BaselIII2019}, has received a lot of focus in risk-averse RL \citep{Prashanth2022Risk}.

More specifically, the \textit{static} CVaR evaluation consists of computing the CVaR of a policy's cumulative random return. Unfortunately, optimizing a policy w.r.t. the static CVaR objective in MDPs has proven to be quite challenging since CVaR suffers from time inconsistency~\citep{Pflug2016Time,Gagne2021Two} and optimal policies may be history-dependent~\citep{Shapiro2014Lectures}. To tackle these problems, prior work has considered dynamic programs (DPs) applied on augmented state spaces in both the primal and dual representations of risk measures. Working under the primal representation, states are augmented by keeping track of the running cumulative return~\citep{Boda2004Stochastic,Bauerle2011Markov,Chow2014Algorithms,Chow2018Risk}. However, such primal-based methods are considered practically inefficient since they require computing the value function on an unbounded continuous state space~\citep{Chow2015Risk,Chapman2021Risk,Li2022Quantile}. The dual representation has therefore been identified as a promising direction, where the sequential decomposition of risk measures can be leveraged~\citep{Chow2015Risk,Pflug2016Time}. States here are augmented by keeping track of the current risk level (between 0 and 1). Although the resulting augmented state space is still continuous, discretization can be applied efficiently since risk levels are bounded~\citep{Chow2015Risk,Li2022Quantile}. Hence, \citet{Chow2015Risk} proposed a Value Iteration (VI) procedure for CVaR in the dual representation, which served as the basis for many later developments in the field~\citep{Chow2015Risk,Chapman2019Risk,Stanko2019Risk,Chapman2021Risk,Rigter2021Risk,Ding2022Cvar,Ding2022Sequential}, until \citet{Hau2023Dynamic} showed counterexample MDPs where this procedure fails to recover the optimal policy. However, from these few empirical results alone, it is not possible to understand \textit{what} makes CVaR VI fail, and therefore \textit{how}, or even \textit{if}, it can be fixed.

\paragraph{Contributions}
The main goal of this work is to diagnose the root causes of recently observed failures in the static dual CVaR DP decomposition. To this end, we first establish a formal analysis framework that recasts the static CVaR evaluation and its DP decomposition as two distinct optimization problems over perturbations. Leveraging this perspective, we identify a set of \textit{risk-assignment consistency constraints} that must be satisfied for the DP evaluation of a policy to be accurate. We show that the previously observed cases where the DP decomposition returned a suboptimal policy are explained by these constraints being inconsistent for the returned policy. 
We then leverage this constraint-based view to present an MDP where the action constraints required for optimality at different initial risk levels are irreconcilable. This proves that no single risk-dependent policy can be \emph{uniformly optimal}, revealing a fundamental limitation of the pursuit of a single, universally optimal policy for all risk levels via the dual decomposition, independent of the DP algorithm used. Practically, these findings suggest that the standard approach of training a single universal policy on the risk-augmented state space is structurally flawed. Our results indicate that practitioners should instead favor training specific policies for targeted risk profiles, akin to primal-based decomposition methods.
To increase readability, we postpone most of the proofs to the appendix. 


\section{Static Risk-Averse Reinforcement Learning}
\label{sec:problem_setting}

Following \citet{puterman2014markov}, let us define a finite Markov decision process (MDP) by a tuple $(\states, \actions, P, \mathcal R, s_0, \gamma)$ where $\mathcal S$ is a finite state space, $\mathcal A$ is a finite action space, $P:\states \times \actions \mapsto \Delta(\states)$ is the transition function between states ($\Delta$ denoting the probability simplex), $\mathcal R:\states \times \actions \times \states \mapsto [0, R_{\text{max}}]$ is the reward function, $s_0 \in \states$ is the initial state from which the process begins, and $\gamma \in [0, 1)$ is a discount factor. At each time step $t \in \mathbb N_0$ of a trajectory, an agent performs an action $A_t \in \actions$ in the current state $S_t\in \states$ according to some decision policy $\pi$. This leads to a transition into the state $S_{t+1}$ sampled from $P(S_t, A_t)$, following which the agent receives the reward $R_{t+1} = \mathcal R(S_t, A_t, S_{t+1})$. This process is repeated over an horizon of $T$ time steps. 
Throughout this paper, we adopt the convention of using uppercase letters to distinguish objects subject to random realization, such as referring to the state $S_t=s \in \mathcal S$. A table of notations can be found in Appendix~\ref{app:notation}.
\begin{remark}
Imposing a deterministic reward function and initial state is done with minimal loss of generality. Any MDP with a stochastic reward function or initial state can be converted to an equivalent MDP with deterministic counterparts, provided the distributions are discrete. This transformation involves augmenting the state space, with the initial action having no effect on the state transition \citep{Sutton2018Reinforcement}.
\end{remark}

\begin{assumptions}
\label{assumptions}
Throughout this work, we consider finite state and action spaces ($|\states| < \infty, |\actions| < \infty$). We focus on the finite-horizon setting with horizon $T$ and we include a discount factor $\gamma \in [0,1)$ to remain consistent with general formulations. Crucially, we restrict our attention to \emph{deterministic} policies. This restriction is standard in the dual CVaR decomposition literature~\citep{Chow2015Risk,Hau2023Dynamic} to isolate \textit{external risk} (stochasticity of the environment) from \textit{internal risk} (policy randomization), simplifying the analysis of the risk level updates.
\end{assumptions}

\paragraph{Policies} We consider agent policies as deterministic action-selection mechanisms. The most general form of policies is \textit{history-dependent} policies $\pi_h: \mathcal H \to A$, where $\mathcal H$ is the set of histories defined as all previous states and actions encountered prior to the current action-selection. Let $H_t \coloneq (S_0, A_0, S_1, A_1, S_2, \dots, S_t)$ and $\mathcal H_t$ respectively denote the current history and the set of possible histories at time $t$. At time $t=0$, the set of histories is limited to the initial state, that is $\mathcal H_0 \coloneq \{(s_0)\}$. For subsequent steps it is defined recursively as the combination of possible previous histories with the previous action and current state concatenated, that is $\mathcal H_{t+1} \coloneq \mathcal H_t \times \actions \times \states$. 
Allowing a slight abuse of notation, we extend transition and action-selection dynamics to histories and define the probability of observing a given history $H_t$ given policy $\pi_h$ as
\begin{align}
\label{eq:prob_history}
P^{\pi_h}(H_t)\coloneq\prod_{\tau=0}^{t-1} P(S_{\tau+1}|S_\tau,A_\tau) \ind{A_\tau=\pi_h(H_\tau)},
\end{align}
where $\mathds{1}$ denotes the indicator function. 

Because of the degree of complexity brought upon by policies operating on the possibly immense set of histories, we are often interested in Markovian policies $\pi:\states \to \actions$, a special case of history-dependent policies where actions are selected only based on the current state. Hereafter, $\pi$ will be used to denote Markovian policies while $\pi_h$ will denote history-dependent policies.


\paragraph{Standard objective} The return associated with a history $H$ is defined as the discounted sum of rewards
\begin{align}
\label{eq:r_tau}
    \mathcal R^H_{0:T} \coloneq \sum_{t=0}^{T-1}\gamma^{t} R_{t+1},
\end{align}
where 
$R_{t+1}=\mathcal R(S_t, A_t, S_{t+1})$.
Given that trajectories in an MDP are generated using a random process, due to state transitions being stochastic, we denote the random return of a trajectory generated by policy $\pi_h$ as a random variable $Z^{\pi_h}$, taking value $\mathcal R^H_{0:T}$ (Eq.~\ref{eq:r_tau}) with $H \sim P^{\pi_h}$ (Eq.~\ref{eq:prob_history}). The standard RL objective~\citep{Sutton2018Reinforcement} is to identify the optimal policy $\pi\opt_h$ that maximizes the \emph{expected return} over histories
\begin{align}
\label{eq:standard_rl_objective}
    \pi\opt_h \in \argmax_{\pi_h} \mathbb E[Z^{\pi_h}].
\end{align}
It is known that Equation~\ref{eq:standard_rl_objective} can always be solved by a Markovian policy $\pi$~\citep{szepesvari2022algorithms}. Unfortunately, because the expectation only weights random outcomes according to their likelihood without taking their value into account, the optimal policy according to this objective may lead to \textit{catastrophic} outcomes over some trajectories~\citep{Mannor2007Bias}. In critical applications where such trajectories should be avoided, one may optimize a risk-averse objective instead, where large negative outcomes are assigned higher importance.

\subsection{Static CVaR for risk aversion}

Let $Z\in\mathbb R$ denote a bounded variable on a probability space $(\Omega, \mathcal{F}, \mathbb{P})$, with cumulative distribution function $F_Z(z)=\mathbb{P}[Z \leq z]$ for some threshold $z \in \Real$. Denote the \emph{Value-at-Risk} (VaR) at risk level $\alpha \in (0, 1]$ as $\text{VaR}_\alpha[Z]\coloneq \min \left\{z \, | \, F_Z(z) \geq \alpha\right\}$.
Assuming that $Z$ represents a payoff that should be maximized, the \emph{conditional-value-at-risk (CVaR)}~\citep{rockafellar2000optimization,Follmer2016Stochastic} at risk level  $\alpha$ is given by
\begin{equation}
\label{eq:def_cvar}
    \cvar{\alpha}{Z} \coloneq \frac{1}{\alpha} \int_{0}^\alpha \text{VaR}_\beta(Z) \text{d}\beta = \underbrace{\inf_{\xi \in \cvarEnv{\alpha}{\mathbb{P}}} \mathbb E_{\xi}\left[Z\right]}_{\text{dual formulation}},
\end{equation}
where $\cvarEnv{\alpha}{\mathbb{P}} \coloneq \left\{\xi: \omega \mapsto \left[0, \frac{1}{\alpha}\right]  \, \Big| \, \int_{\omega \in \Omega} \xi(\omega)\mathbb{P}(\omega) d\omega = 1 \right\}$ defines the CVaR$_\alpha$ \emph{risk envelope} around distribution $\mathbb{P}$ 
and $\mathbb E_\xi[Z]$ is the $\xi$-reweighed expectation of $Z$. If $Z$ has a continuous distribution, it is well known that we have $\text{CVaR}_\alpha[Z]=\mathbb E \left[Z | Z \leq \text{VaR}_\alpha[Z] \right]$, which can be interpreted as the expected value of the worst $\alpha$ outcomes of $Z$. Note that CVaR is monotonically increasing in $\alpha$ with edge cases representing $\cvar{0}{Z}=\ess \inf [Z]$ and $\cvar{1}{Z}=\mathbb E[Z]$.

The \textit{dual formulation} in Equation~\ref{eq:def_cvar} shows that the CVaR can be expressed as an optimization problem, where the objective is to find perturbations $\xi$ applied to the stochastic generative process of variable $Z$ such as to minimize its expectation. From the definition of the 
risk envelope $\cvarEnv{\alpha}{\mathbb{P}}$, we can observe that the perturbations enjoy two interesting properties. First, because $\xi$ represents \textit{multiplicative} interventions on an event's likelihood, it only affects events with nonzero probability. Also, because the largest magnitude of perturbations on an event is $\frac{1}{\alpha}$, one can view $\frac{1}{\alpha}$ as a \textit{perturbation budget} which is naturally minimal at $\alpha=1$ and increases as $\alpha$ decreases to $0$, simultaneously recovering the monotonically increasing property of CVaR and its edge cases. 

\paragraph{CVaR-RL objective} Recalling that the random return of policy $Z^{\pi_h}$ is a random variable, one can therefore define the static CVaR of a policy as
\begin{align}
\label{eq:static_cvar_policy_eval}
\cvar{\alpha}{Z^{\pi_h}} \coloneq \min_{\xi \in \cvarEnv{\alpha}{P^{\pi_h}}} \sum_{H \in \mathcal H_T} P^{\pi_h}(H) \xi(H) \mathcal R_{0:T}^H,
\end{align}
where we shall hereafter refer to $\xi$ as \emph{history perturbations}. By emphasizing negative outcomes, Equation~\ref{eq:static_cvar_policy_eval} naturally yields the CVaR-RL risk-averse objective~\citep{tamar2015optimizing}
\begin{align}
\label{eq:cvar_rl_problem}
    \pi_h\opt \in \argmax_{\pi_h} \cvar{\alpha}{Z^{\pi_h}},
\end{align}
where one aims to instead find a policy maximizing the CVaR$_\alpha$ of its random return. Because CVaR$_\alpha[Z^{\pi_h}]$ can be intuitively interpreted as the expectation of the worst $\alpha$ fraction of trajectories when following policy $\pi_h$, optimizing the CVaR-RL objective should yield policies less prone to catastrophic outcomes than the standard RL objective~(Eq.~\ref{eq:standard_rl_objective}), with a lower $\alpha$ leading to increased cautiousness. 

\subsection{CVaR-RL dynamic decomposition}


Trajectory-level computation of static CVaR  (Eq.~\ref{eq:static_cvar_policy_eval}) is impractical because it requires computing $P^{\pi_h}$ for all trajectories, which can be prohibitive for large state and action spaces. Fortunately, the CVaR decomposition Theorem~\citep{Chow2015Risk,Pflug2016Time} grants a recipe for expressing the evaluation at state-level.
\begin{theorem}[CVaR decomposition, Thm. 2 from \citet{Chow2015Risk}]
 For any time step $t \geq 0$, denote by $Z^{\pi_h}_{t:T}$ the return from time $t+1$ onward under history-dependent policy $\pi_h$. Given current history $H_t$, the $\text{CVaR}_\alpha$ of $Z^{\pi_h}_{t:T}$ obeys the following decomposition:
\begin{align*}
    \cvar{\alpha}{Z_{t:T}^{\pi_h} \,|\,H_t} = \min_{\tilde\xi\in\cvarEnv{\alpha}{P(\cdot | S_t, A_t)}} \sum_{s'\in\states} P(s' | S_t, A_t) \tilde\xi(s') \cvar{\alpha \cdot \tilde \xi(s')}{Z_{t:T}^{\pi_h}  \, | \, H'},
\end{align*}
where $\tilde \xi$ are perturbations over next state transitions, action $A_t$ is given by policy $\pi_h(H_t)$, and $H'=H_t \cup (A_t, s')$ is a possible history realization at time $t+1$.
\label{thm:cvar_decomposition}
\end{theorem}
%
\begin{remark}
    We distinguish perturbations over next states ($\tilde \xi$) from perturbations over histories ($\xi$). While both perturbations impact the sampling of events, $\tilde \xi$ also updates the ongoing risk-level as dictated by the CVaR decomposition theorem (Thm.~\ref{thm:cvar_decomposition}). When accumulated over an entire history, state perturbations implicitly yield history-level perturbations, but the connection between the two perturbation levels is complex and will be a core component of our analysis.
\end{remark}
\paragraph{Risk-dependent policies} Theorem~\ref{thm:cvar_decomposition} shows that the $\text{CVaR}_\alpha$ at any given time $t$ can be expressed as combination of $\text{CVaR}_{\alpha'}$ values of possible next states $s'$ for updated risk levels $\alpha' = \alpha\cdot\tilde\xi(s')$ at time $t+1$. The running risk level 
therefore contains all the information necessary to compute the CVaR of a history-dependent policy $\pi_h$. This motivated \citet{Chow2015Risk} to introduce the \emph{risk-augmented} state space $\tilde \states : \states \times (0, 1]$, defined for any state $s \in \states$ and risk level $y \in (0, 1]$, and the corresponding \emph{risk-dependent} Markovian policies on the augmented state space $\tilde\pi: \tilde \states \to \actions$. \citet{Chow2015Risk} suggested that operating over $\tilde{\mathcal S}$ would suffice to retrieve the optimal history-dependent policy and evaluate its corresponding static CVaR through a \emph{value function} based on the mechanism of Theorem~\ref{thm:cvar_decomposition}.

\begin{definition}[Risk-dependent-policy value function]
\label{def:dp_cvar_eval}
The value function $\V{\tilde\pi}{}(s,y)$ of any risk-dependent policy $\tilde \pi$ is the solution to
\begin{align}
\label{eq:dp_cvar_eval}
\V{\tilde\pi}{t+1}(s, y) = \min_{\tilde\xi\in\cvarEnv{y}{P(\cdot | s, a)}} \sum_{s'\in\states} P(s' | s, a) \tilde\xi(s') \left[\mathcal R(s,a,s') + \gamma \V{\tilde\pi}{t}(s', y')\right]
\end{align}
where $a=\tilde\pi(s,y)$ is the action selected by the policy, $y'=y\cdot\tilde\xi(s')$ is the subsequent risk level, and $\V{\tilde\pi}{0}(s,y)=0$ for all states $s \in \states$ and risk levels $y \in (0, 1]$. We let $\V{\tilde\pi}{} \coloneq \V{\tilde\pi}{T}$.
\end{definition}
Crucially, any risk-dependent policy $\tilde\pi$ induces a corresponding history-dependent policy $\tilde\pi_h^\alpha$ for a given initial risk level $Y_0 = \alpha$.
At time $t$, given the risk-augmented state $(S_t,Y_t)$ and the selected action $A_t = \tilde\pi(S_t, Y_t)$, the subsequent risk level is updated to $Y_{t+1} = Y_t \cdot \tilde\xi\opt(S_{t+1}|S_t, Y_t, A_t)$, where the optimal perturbations $\tilde\xi\opt$ denote the solution to the value function $\V{\tilde\pi}{}(S_t,Y_t)$.
Repeating this process $t$ times allows to compute the action $\tilde \pi_h^\alpha(H_t)$ for any history $H_t$. In the remainder of this paper, we slightly abuse terminology and refer to a risk-dependent policy's static CVaR to represent the static CVaR of its history-dependent counterpart.


In light of this correspondence between risk-dependent and history-dependent policies, \citet{Chow2015Risk} proposed a Value Iteration algorithm, which we refer to as \emph{CVaR VI}, to find the risk-dependent policy with the optimal value function $\tilde\pi^\star \in \arg\max_{\tilde\pi} \V{\tilde\pi}{}(s_0, \alpha)$. Tentative proofs \citep{Chow2015Risk,Li2022Quantile} claimed $\tilde\pi^\star$ represented a risk-dependent version of the CVaR-optimal history-dependent policy $\pi_h\opt$, hence presenting a dynamic program decomposition of the CVaR-RL objective (Eq.~\ref{eq:cvar_rl_problem}). The optimality of the policy returned by CVaR VI was however refuted by \citet{Hau2023Dynamic}, who presented a counterexample MDP where the algorithm returns a suboptimal policy.

\paragraph{CVaR evaluation gap} 
In this work, we aim to explain \emph{why} CVaR VI fails at a more fundamental level. 
To this end, we focus on the root cause, that is the discrepancy between the value function of a risk-dependent policy and its corresponding static CVaR, which we formally define as the \emph{CVaR evaluation gap}
\begin{align}
\label{eq:cvar_evaluation_gap}
    \V{\tilde\pi}{}(s_0, \alpha) - \cvar{\alpha}{Z^{\tilde\pi_h^\alpha}}.
\end{align}
A positive gap indicates that the value function of the risk-dependent policy overestimates its history-dependent counterpart's true CVaR. 

\section{An Explicit Mapping from the Value Function to the Static CVaR}
\label{section:two_optim_problems}





To diagnose the source of the CVaR evaluation gap (Eq.~\ref{eq:cvar_evaluation_gap}), we first formalize the relationship between the value function of a risk-dependent policy (Eq.~\ref{eq:dp_cvar_eval}) at the initial state $(s_0,\alpha)$ and its corresponding static CVaR. We find that both problems can be cast as distinct, but closely related, perturbation optimization problems. Although they operate over different optimization spaces, respectively state-level perturbations $\tilde\xi$ and history-level perturbations $\xi$, we can derive a formal mapping from one problem to the other. Exploiting the mapping's properties, we then establish that the value function of a risk-dependent policy constitutes an upper bound on the static CVaR of the policy.

While the static CVaR evaluation (Eq.~\ref{eq:static_cvar_policy_eval}) is inherently defined as a minimization problem, evaluating the value function of a risk-dependent policy (Eq.~\ref{eq:dp_cvar_eval}) is defined as requiring $T$ recursive minimization problems, 
making for more cumbersome mathematical manipulations.
In order to ease the manipulation of the latter, we first define the value function of a risk-dependent policy under a \emph{fixed} set of \emph{state-level perturbations}. For a fixed risk-dependent policy $\tilde \pi$, let $\bm{\tilde{\xi}}$ denote a complete specification of state-level perturbations such that $\bm{\tilde \xi}(\cdot | s, y, a) \in \cvarEnv{y}{P(\cdot|s,a)}$ for all states $s \in \states$, risk levels $y \in (0, 1]$, and actions $a \in \actions$. We further define $\bm{\tilde\Xi}\coloneq \left\{ \bm{\tilde \xi} : \states \times (0, 1] \times \actions \to \states \times \mathbb R^+ \,  | \, \bm{\tilde \xi}(\cdot | s, y, a) \in \cvarEnv{y}{P(\cdot|s,a)} \forall (s, y, a) \in \states \times (0, 1] \times \actions \right\}$ as the set of all such valid state-level perturbations.


\begin{definition}[Policy-perturbations value function]
\label{def:policy_pertub_value_function}
The \emph{policy-perturbations value function} of a risk-dependent policy $\tilde\pi$ under state-level perturbations $\bm{\tilde\xi} \in \bm{\tilde{\Xi}}$ is obtained via:
\begin{align}
\label{eq:policy_perturb_value_function}
\V{\tilde\pi, \bm{\tilde\xi}}{t+1}(s, y) = \sum_{s'\in\states} P(s' | s, a)\, \bm{\tilde\xi}(s' | s, y, a) \left[R(s, a, s') + \gamma  \V{\tilde\pi, \bm{\tilde\xi}}{t}(s', y') \right],
\end{align}
where we used action $a = \tilde \pi (s, y)$, updated risk level $y' = y \cdot \bm{\tilde\xi}(s' | s, y, a)$, and $\V{\tilde\pi, \bm{\tilde\xi}}{0}(s, y)=0$ for all states $s \in \states$ and risk levels $y \in (0, 1]$. We let $\V{\tilde\pi, \bm{\tilde\xi}}{}(s, y)\coloneq\V{\tilde\pi, \bm{\tilde\xi}}{T}(s, y)$. 
\end{definition}
This definition differs from the one in Equation~\ref{eq:dp_cvar_eval} because it concerns \emph{fixed} state-level perturbations $\bm{\tilde \xi}$ instead of computing them recursively at every iteration. The value function of a risk-dependent policy $\tilde\pi$ can now be seen as finding the best possible perturbations $\bm{\tilde\xi} \in \bm{\tilde{\Xi}}$ to minimize this value.





\begin{restatable}[Value function evaluation]{lemma}{lemmaDpOptimProblem}
\label{lemma:value_function_evaluation}
Under the conditions of Assumptions~\ref{assumptions}, the value function evaluation of a risk-dependent policy $\tilde\pi$ (Eq.~\ref{eq:dp_cvar_eval}) is equivalent to solving
\begin{align}
\label{eq:dp_optim_problem}
\V{\tilde\pi}{}(s, y) = \min_{\bm{\tilde\xi} \in \bm{\tilde\Xi}} \V{\tilde\pi,\bm{\tilde\xi}}{}(s, y), 
\end{align}
where the above holds for all state-risk level pairs $(s, y)$ simultaneously, meaning a single state-level perturbations set $\bm{\tilde\xi}\opt$ is optimal for all $(s,y)$.
\end{restatable}

\begin{remark}
It is important to note that $\bm{\tilde\Xi}$ is the set of state-level perturbations $\bm{\tilde\xi}$ for the \textit{entire} MDP, meaning that any difference in perturbation of a single (state, action, risk-level) tuple creates a different $\bm{\tilde{\xi}}’$. Therefore, taking the minimum over $\bm{\tilde{\Xi}}$ corresponds to selecting the worst possible simultaneous perturbation of the \textit{entire} MDP for a given policy.
\end{remark}

We now have two distinct single-step optimization problems for static CVaR evaluation: the static evaluation over history perturbations $\xi$ (Eq.~\ref{eq:static_cvar_policy_eval}) and the value function evaluation over state-level perturbations $\bm{\tilde\xi}$ (Eq.~\ref{eq:dp_optim_problem}). The two problems are in fact intimately connected. That is, any state-level perturbations $\bm{\tilde\xi}$ can be mapped to corresponding history-level perturbations $\xi$ by taking the product of state-level perturbations along each history. For an initial risk level $\alpha$ and history $H \in \mathcal H_T$, we define this mapping as
\begin{align*}
    \zeta_\alpha^{\bm{\tilde\xi}}(H)\coloneq \prod_{t=0}^{T-1} \bm{\tilde\xi}(S_{t+1} | S_t, Y_t, A_t),
\end{align*}
where the risk levels $Y_t$ are incremented following $\bm{\tilde\xi}$ and starting from $Y_0=\alpha$. We now show that the mapping $\zeta_\alpha^{\bm{\tilde\xi}}$ produces valid history perturbations recovering the value function for $\tilde \pi$ at the history-level.

\begin{restatable}[State-level perturbations evaluation correspondence]{proposition}{propositionConnexionDpCvar}
\label{proposition:connexion_dp_cvar}
Under the conditions of Assumptions~\ref{assumptions}, for any risk-dependent policy $\tilde \pi$, initial risk level $\alpha$, and state-level perturbations $\bm{\tilde\xi} \in \bm{\tilde\Xi}$, the mapping $\zeta_\alpha^{\bm{\tilde\xi}}$ produces valid history perturbations, that is $\zeta_\alpha^{\bm{\tilde\xi}} \in \cvarEnv{\alpha}{P^{\tilde\pi_h^\alpha}}$, for which we have
\begin{align*}
\sum_{H \in \mathcal H_T} P^{\tilde\pi_h^\alpha}(H)\, \zeta_\alpha^{\bm{\tilde\xi}}(H)\, \mathcal R_{0:T}^H = \V{\tilde\pi,\bm{\tilde\xi}}{}(s_0, \alpha).
\end{align*}
\end{restatable}
%
Because Proposition~\ref{proposition:connexion_dp_cvar} applies to \emph{any} state-level perturbations set $\bm{\tilde\xi} \in \bm{\tilde\Xi}$, in particular it applies to the optimal state-level perturbations set $\bm{\tilde\xi}\opt \in \arg\min_{\bm{\tilde\xi} \in \bm{\tilde\Xi}} \V{\tilde\pi,\bm{\tilde\xi}}{}(s_0,\alpha)$ for a fixed $\tilde \pi$. It follows that the value function evaluation (Eq.~\ref{eq:dp_optim_problem}) is always an upper-bound to the true static CVaR of a policy (Eq.~\ref{eq:static_cvar_policy_eval}).


\begin{corollary}[Static CVaR upper-bound]
\label{cor:dp_geq_static_eval}
Under the conditions of Assumptions~\ref{assumptions}, for any risk-dependent policy $\tilde \pi$ and initial risk level $\alpha$, we have
\begin{align*}
\cvar{\alpha}{Z^{\tilde\pi_h^\alpha}} \leq 
\V{\tilde\pi}{}(s_0, \alpha).
\end{align*}    
\end{corollary}
%
%
As a result of Corollary~\ref{cor:dp_geq_static_eval}, the CVaR evaluation gap (Eq.~\ref{eq:cvar_evaluation_gap}) is non-zero if and only if the optimal history perturbations $\xi^\star$ are not in the image of the mapping $\zeta_\alpha$. That is, if the best global (history-level) perturbations cannot be decomposed into a sequence of valid local (state-level) perturbations, the value function evaluation will return an erroneous estimation of the policy's static CVaR. 

\section{Characterizing the CVaR Evaluation Gap}
\label{sec:evaluation_gap}

We established that the value function of a risk-dependent policy provides an upper bound on its true static CVaR. In this section, we now investigate the exact conditions under which the upper bound is strict, leading to a CVaR evaluation gap. 
We show that the gap emerges when the optimal history-level perturbations are not \emph{realizable} at the state level, a property that we formalize through a set of consistency constraints.




\begin{definition}[Realizable trajectory perturbations]
For a given risk-dependent policy $\tilde{\pi}$ and initial risk level $\alpha \in (0, 1]$, trajectory perturbations $\xi \in \cvarEnv{\alpha}{P^{\tilde\pi_h^\alpha}}$ are \emph{realizable} if there exists state-level perturbations $\bm{\tilde\xi} \in \bm{\tilde\Xi}$ such that $\zeta_\alpha^{\bm{\tilde\xi}}(H) = \xi(H)$ for all histories $H \in \mathcal H$.
\end{definition}
%
%
%
The existence of such state-level perturbations $\bm{\tilde\xi}$ hinges on our ability to define a sequence of intermediate risk levels $Y_t$ that are mutually consistent for all histories. Recall that $\mathcal H\coloneq\bigcup_{t=1}^T \mathcal H_t$ is the set off all possible histories of length $|H| \leq T$ and define $H_{0:k}\coloneq (S_0, A_0, \dots, S_k) \in \mathcal H_k$ as the $k$-length subsequence of a given history $H$. 
We formalize the mutual consistency notion by defining a \textit{risk level assignment} $\mathcal{Y}: \mathcal{H} \to (0, 1]$ that maps any history $H$ to a risk level $Y$. A risk level assignment $\mathcal Y$ is consistent with respect to trajectory perturbations $\xi$ if it enforces the correct total perturbations on all histories, while also respecting all stepwise constraints on the CVaR risk envelope and maintaining the correct sampled action from the risk-dependent policy. We will refer to these sets of constraints as the \emph{risk-assignment consistency constraints}.  


\begin{definition}[Risk assignment consistency constraints]
\label{def:risk_level_assignment_constraints}
For a given risk-dependent policy $\tilde\pi$, initial risk level $\alpha \in (0, 1]$, and history perturbations $\xi \in \cvarEnv{\alpha}{P^{\tilde\pi_h^\alpha}}$, a risk level assignment $\mathcal{Y}$ is \emph{consistent} if, for all histories $H \in \mathcal H_T$ with $P^{\tilde\pi_h^\alpha}(H)>0$, it satisfies the following constraints:
\begin{enumerate}
    \item \textbf{Risk propagation:} The assignment must propagate risk according to $\xi$, that is $\mathcal{Y}(H_{0:0}) = \alpha$ and $\mathcal{Y}(H) = \alpha \cdot \xi(H)$.
    \item \textbf{State-level risk envelope:} For $t \in \{0, \dots, T-1\}$, the risk envelope constraint over states must be respected for all possible states:
    \begin{align*}
    \sum_{s' \in \states} P(s'|S_t, A_t) \frac{\mathcal{Y}(H_{0:t} \cup (A_t,s'))}{\mathcal{Y}(H_{0:t})} = 1.
    \end{align*}
    \item \textbf{Action-selection consistency:} For $t \in \{0, \dots, T-1\}$, the actions taken in the history must match the risk-dependent policy's output for the assigned risk level:
    \begin{align*}
        \tilde\pi(S_t, \mathcal{Y}(H_{0:t})) = A_t.
    \end{align*}
\end{enumerate}
\end{definition}



We are now prepared to formally connect the risk-assignment consistency constraints with the realizability property introduced earlier.

\begin{restatable}[Consistency if and only if realizability]{lemma}{lemmaRealizabilityIffConsistency}
\label{lemma:realizability_iff_consistency}
Under the conditions of Assumptions~\ref{assumptions}, , for any risk-dependent policy $\tilde\pi$ and initial risk level $\alpha \in (0, 1]$, history perturbations $\xi \in \cvarEnv{\alpha}{P^{\tilde\pi_h^\alpha}}$ are realizable if and only if there exists a consistent risk level assignment $\mathcal{Y}$ such that all  risk-assignment consistency constraints (Def.~\ref{def:risk_level_assignment_constraints}) hold simultaneously.
\end{restatable}

The difficulty in satisfying the risk-assignment consistency constraints (Def.~\ref{def:risk_level_assignment_constraints}) lies in finding an assignment $\mathcal{Y}$ that satisfies all three constraint sets \textit{simultaneously}. While it is clear that each constraint set can be satisfied in isolation, their intersection may be empty. This tension is the fundamental source of the CVaR evaluation gap (Eq.~\ref{eq:cvar_evaluation_gap}), which we formalize in the following theorem.

\begin{restatable}[Conditions for CVaR evaluation gap]{theorem}{thmAbsenceOfGap}
\label{thm:absence_of_gap}
Under the conditions of Assumptions~\ref{assumptions}, for any risk-dependent policy $\tilde\pi$ and initial risk level $\alpha \in (0, 1]$, we have $\cvar{\alpha}{Z^{\tilde\pi_h^\alpha}} = \V{\tilde\pi}{}(s_0, \alpha)$ if and only if there exists at least one set of optimal history perturbations $\xi^\star$ solution to the static CVaR evaluation (Eq.~\ref{eq:static_cvar_policy_eval}) such that the risk-assignment constraints (Def.~\ref{def:risk_level_assignment_constraints}) can be satisfied simultaneously. 
\end{restatable}

Theorem~\ref{thm:absence_of_gap} presents a formal characterization of necessary and sufficient conditions for when a CVaR evaluation gap occurs. It provides the valuable insight that, hidden under the mismatch between a risk-dependent policy's value function and its static CVaR lies an unsolvable constraint satisfaction problem on the risk level evolution. More specifically, risk-dependent policies can induce action-selection consistency constraints that cannot be satisfied simultaneously with the other risk propagation and state-level risk envelope requirements, hampering the evaluation of a policy's true CVaR.

%
%
%
The proposed constraints satisfaction perspective also clarifies why the evaluation is always accurate for Markovian policies (Thm. 3.1 in \citet{Hau2023Dynamic}). For such policies, the action-selection consistency constraints are non-binding, as the policy does not depend on the risk level. A consistent risk-assignment can therefore always be constructed by recursively applying the CVaR decomposition theorem (Thm.~\ref{thm:cvar_decomposition}), guaranteeing the absence of a CVaR evaluation gap, as detailed in the following corollary.


\begin{restatable}[Existence of corresponding risk-dependent policy]{corollary}{corHistPolNoGap}
\label{corollary:markovian_policies_no_gap}
Under the conditions of Assumptions~\ref{assumptions}, for any Markovian policy $\pi:\mathcal H\to \actions$ and initial risk level $\alpha \in (0, 1]$, there exists a risk-dependent policy $\tilde\pi$ such that $\cvar{\alpha}{Z^{\pi}} = \V{\tilde\pi}{}(s_0, \alpha)$.
\end{restatable}

The biggest issue that stems from Theorem~\ref{thm:absence_of_gap} and Corollary~\ref{corollary:markovian_policies_no_gap} is that one cannot in general guarantee that \emph{all} risk-dependent policies will have optimal history perturbations $\xi\opt$ that have consistent risk-assignment constraints (Def.~\ref{def:risk_level_assignment_constraints}). As a result, the set of all risk-dependent policies may contain policies with a positive CVaR evaluation gap who will have an inaccurately high $\V{\tilde\pi}{}(s_0,\alpha)$, impeding on the optimality of algorithms searching for $\tilde\pi\opt \in \arg\max_{\tilde\pi} \V{\tilde\pi}{}(s_0,\alpha)$ like CVaR VI~\citep{Chow2015Risk}. We now present a worked example where the optimal risk-dependent policy has a positive CVaR evaluation gap and is therefore suboptimal.


\subsection*{A deeper dive into the counterexample from \citet{Hau2023Dynamic}}

\def\rewardcolor{darkspringgreen}
\def\transitioncolor{darkslateblue}
\begin{figure}[t]
  \centering
  \begin{tikzpicture}[->,>=stealth',shorten >=1pt,node distance=1.8cm,semithick,level distance=23mm]
    \tikzstyle{level 1}=[sibling distance=25mm]
    \tikzstyle{level 2}=[sibling distance=15mm]
    \tikzstyle{level 3}=[sibling distance=9mm]
    \tikzstyle{level 4}=[sibling distance=5mm]

    \node (s0){$s_0, a_1$} [grow'=right,->]
    child {
      node (s1) {$s_1$}
      child {
        node (s11) [anchor=north west] {$s_1,a_1$}
        child {
          node (s3) {$s_3$}
          child {
            node[align=right,text width=5ex] (s111) { \textcolor{\rewardcolor}{$600$} }
          }
          edge from parent node[pos=0.5,fill=white, node font=\tiny] {\textcolor{\transitioncolor}{$0.75$}}
        }
        child {
          node (s4) {$s_4$}
          child {
            node[align=right,text width=5ex] (s112) { \textcolor{\rewardcolor}{$-600$} }
          }
          edge from parent node[pos=0.5,fill=white,node font=\tiny] {\textcolor{\transitioncolor}{$0.25$}}
        }
      }
      child {
        node (s12) [anchor=west] {$s_1, a_2$}
        child {
          node (s5) {$s_5$}
          child {
            node[align=right,text width=5ex] (s121) {\textcolor{\rewardcolor}{$0$}}
          }
        }
      }
      child {
        node (s13) [anchor=south west] {$s_1, a_3$}
        child {
          node (s6) {$s_6$}
          child {
            node[align=right,text width=5ex] (s131) {\textcolor{\rewardcolor}{$-100$}}
          }
          edge from parent node[pos=0.5,fill=white,node font=\tiny] {\textcolor{\transitioncolor}{$0.5$}}
        }
        child {
          node (s7) {$s_7$}
          child {
            node[align=right,text width=5ex] (s132) {\textcolor{\rewardcolor}{$400$}}
          }
          edge from parent node[pos=0.5,fill=white,node font=\tiny] {\textcolor{\transitioncolor}{$0.5$}}
        }
      }
      edge from parent node[pos=0.5,fill=white, node font=\tiny] {\textcolor{\transitioncolor}{$0.5$}}
    }
    child { node (s2) {$s_2$}
      child {node (s21) [anchor=west] {$s_2,a_1$}
        child {
          node (s8) {$s_8$}
          child {
            node[align=right,text width=5ex] (s211) {\textcolor{\rewardcolor}{$200$} }
          }
        }
      }
      edge from parent node[pos=0.5,fill=white, node font=\tiny] {\textcolor{\transitioncolor}{$0.5$}}
    };
\end{tikzpicture}
    \caption[Sample counterexample MDP from \citet{Hau2023Dynamic}]{Sample MDP from \citet{Hau2023Dynamic}. Next state transition probabilities are in \textcolor{\transitioncolor}{blue} while rewards are in \textcolor{\rewardcolor}{green}.}
    \label{fig:cvar-3actionscounterexample}
\end{figure}
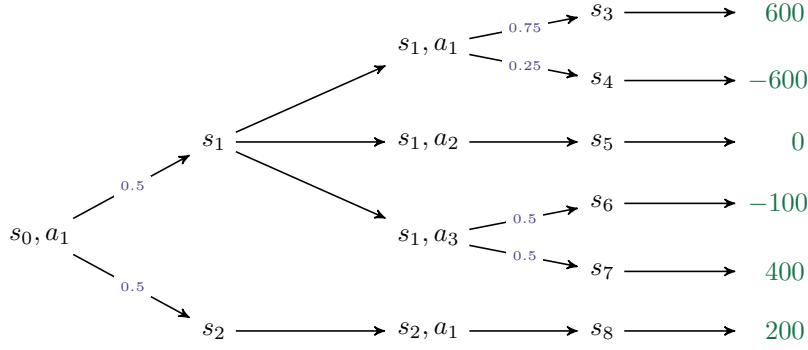

We now apply our newly introduced constraint satisfaction analysis to a counterexample presented in \citet{Hau2023Dynamic}. We use the MDP shown in Figure~\ref{fig:cvar-3actionscounterexample}, with horizon $T=2$ and initial risk $\alpha=0.5$. Note that our MDP differs superficially from the one in \citet{Hau2023Dynamic}, namely because we use a deterministic initial state $s_0$ with a single action available $a_1$ that does not impact the transition to the first state $S_1$, a procedure equivalent to the stochastic initial state presented in the original MDP.
%

For our example, we consider the risk-dependent policy $\tilde{\pi}$ produced by CVaR VI \citep{Chow2015Risk} 
and its optimal state-level perturbations $\bm{\tilde\xi}\opt$ solution to the risk-dependent-policy value function (Eq.~\ref{eq:dp_cvar_eval}):
\begin{align*}
  \tilde\pi(s_1, y) = \begin{cases}
    a_1 & \text{if } y > 0.5 \\
    a_2 & \text{if } y \le 0.5
  \end{cases} &&
  \bm{\tilde\xi}\opt(s' | s, y, a) = 1, \quad \text{for all reachable } (s,y,a).
\end{align*}
The optimal state-level perturbations $\bm{\tilde\xi}\opt$ therefore apply no changes to next states sampled, so the risk level remains $Y_1=0.5$ after the transition from $s_0$. As a result, the corresponding history-dependent policy takes action $a_2$ when reaching $s_1$, that is $\tilde\pi_h^{0.5}\left((s_0,a_1,s_1)\right)=a_2$. Solving the static CVaR evaluation (Eq.~\ref{eq:static_cvar_policy_eval}), we find the corresponding history probabilities and optimal history perturbations $\xi\opt$:
\begin{align*}
    P^{\tilde\pi_h^{0.5}}(H) = \begin{cases} 0.5 & \text{if } H=(s_0,a_1,s_1,a_2,s_5) \\ 0.5 & \text{if } H=(s_0,a_1,s_2,a_1,s_8) \end{cases} &&
    \xi\opt(H) = \begin{cases} 2 & \text{if } H=(s_0,a_1,s_1,a_2,s_5)\\ 0 & \text{if } H=(s_0,a_1,s_2,a_1,s_8) \end{cases}
\end{align*}
For the optimal history perturbations $\xi\opt$ to be realizable, there must exist a consistent risk level assignment $\mathcal{Y}$ such that the risk-assignment consistency constraints (Def.~\ref{def:risk_level_assignment_constraints}) have non-empty intersection. Observing that histories of any length are fully defined by their final state since states never repeat in this MDP, the constraints can be expressed as:
\begin{enumerate}
    \item \textbf{Risk propagation}: Directly applying trajectory perturbations $\xi\opt$, we get
    \begin{align*}
        \mathcal{Y}(s_0) = \alpha = 0.5, && \mathcal{Y}(s_5) = 1, && \text{and} &&  \mathcal{Y}(s_8) = 0.
    \end{align*}
    \item \textbf{State-level risk envelope}: For $t=0$, the constraint is equivalent to $\mathcal{Y}(s_1) +\mathcal{Y}(s_2) = 1$. The constraints for $t=1$ are $\mathcal{Y}(s_5) = \mathcal{Y}(s_1)$ and $\mathcal{Y}(s_8) = \mathcal{Y}(s_2)$. Combining with the risk propagation constraint, we find that the risk assignment at $t=1$ must be 
    \begin{align}
    \label{eq:we_risk_envelope_constraint}
        \mathcal{Y}(s_1) = 1 && \text{and} && \mathcal{Y}(s_2) = 0.
    \end{align}
    \item \textbf{Action-selection consistency}: The history that ends in $s_5$ requires that at state $s_1$, the action $a_2$ was selected. According to the policy $\tilde\pi(s_1,y)$, this is only possible if the risk level satisfies 
    \begin{align}
    \label{eq:we_action_selection_constraint}
        \mathcal{Y}(s_1) \le 0.5.
    \end{align}
\end{enumerate}

Figure~\ref{fig:risk_assignment_impossibility} displays a visual representation of the risk-assignment constraints on the considered MDP given the risk-dependent policy $\tilde\pi$ obtained using CVaR VI and initial risk level $\alpha = 0.5$.
We observe that the state-level risk envelope (Eq.~\ref{eq:we_risk_envelope_constraint}) and the action-selection consistency (Eq.~\ref{eq:we_action_selection_constraint}) constraints are impossible to satisfy simultaneously in state $s_1$.  Thus, no consistent risk level assignment $\mathcal{Y}$ exists for $\xi\opt$. By Theorem~\ref{thm:absence_of_gap}, this confirms a positive CVaR evaluation gap, providing context to the results of~\citet{Hau2023Dynamic}. 

\def\riskpropagationcolor{brown}
\def\rasenvcolor{violet}
\def\actionselectioncolor{teal}
\begin{figure}[t]
  \centering
  \begin{tikzpicture}[->,>=stealth,shorten >=1pt,node distance=1.8cm,semithick,level distance=23mm]
    \tikzstyle{level 1}=[sibling distance=14mm]
    \tikzstyle{level 2}=[sibling distance=14mm]

    \node (s0) [anchor=west, label={left:{
    \begin{tabular}{c}
         $\Big\{\textcolor{\riskpropagationcolor}{\mathcal Y(s_0) = 0.5}\Big\}$ 
    \end{tabular}}}] {$s_0, a_1$} [grow'=right,->]
    child {
      node (s1) [anchor=west, label={above:{%
        \begin{tabular}{c}
          $\Big\{\textcolor{\actionselectioncolor}{\mathcal Y(s_1) = 0.5}\Big\} \cap \Big\{\textcolor{\rasenvcolor}{\mathcal Y(s_1) = 1}\Big\}$
        \end{tabular}%
      }}] {$s_1$} 
      child {
        node (s12) {$s_1, a_2$}
        child {
          node (s5) [anchor=west, label={right:{
            \begin{tabular}{c}
                 $\Big\{\textcolor{\riskpropagationcolor}{\mathcal Y(s_5) = 1}\Big\}$ 
            \end{tabular}}}] {$s_5$}
        }
      }
    }
    child { node (s2) [anchor=west, label={below:{
        \begin{tabular}{c}
             $\Big\{\textcolor{\rasenvcolor}{\mathcal Y(s_2) = 0}\Big\}$ 
        \end{tabular}}}] {$s_2$}
      child {node (s21) {$s_2,a_1$}
        child {
          node (s8) [anchor=west, label={right:{
        \begin{tabular}{c}
             $\Big\{\textcolor{\riskpropagationcolor}{\mathcal Y(s_8) = 0}\Big\}$ 
        \end{tabular}}}] {$s_8$}
        }
      }
    };
\end{tikzpicture}
    \caption[Visual representation of the  risk-assignment constraints on the MDP from Figure~\ref{jmlr_fig:cvar-3actionscounterexample}]{Visual representation of the  risk-assignment constraints on the MDP from Figure~\ref{fig:cvar-3actionscounterexample}, with the policy $\tilde\pi$ obtained using CVaR VI \citep{Chow2015Risk} and setting $\alpha=0.5$. Risk propagation constraints are in \textcolor{\riskpropagationcolor}{brown}, state-level risk envelope constraints are in \textcolor{\rasenvcolor}{purple}, and action selection constraints are in \textcolor{\actionselectioncolor}{teal}.}
    \label{fig:risk_assignment_impossibility}
\end{figure}

\begin{remark}
Not all risk-dependent policies have a CVaR evaluation gap. In Appendix~\ref{appendix:policy_no_evaluation_gap}, we present such a risk-dependent policy on the same MDP from Figure~\ref{fig:cvar-3actionscounterexample} for which the dual-based DP evaluation accurately represents the static CVaR of its history-dependent counterpart. 
\end{remark}

\section{From Impossible Evaluation to Impossible Uniform Optimality}
\label{sec:impossibility_result}


One of the most significant consequence of using the dual-based CVaR decomposition~\citep{Pflug2016Time} and finding a single optimal risk-dependent policy $\tilde \pi$ is that it could simultaneously retrieve the optimal history-dependent policy for \textit{any} initial risk-level~\citep{Chow2015Risk}, a property we will refer to as \emph{uniformly optimality}. This property represented a massive advantage over the primal-based decomposition~\citep{Bauerle2011Markov}, which is constrained to find the optimal policy for a \textit{single} initial risk-level. Leveraging the risk-assignment constraints perspective developed in the previous section, we now show that there exists an MDP where it is impossible for a single risk-dependent policy $\tilde{\pi}$ to be uniformly optimal, proving this inherent property to be impossible in general.
%
%
%
We begin by formalizing the notion of uniform optimality.

\begin{definition}[Uniformly optimal policy]
\label{def:uniform_optimal_policy}
A risk-dependent policy $\tilde\pi$ is \emph{uniformly optimal} if its corresponding history-dependent policy $\tilde\pi_h^\alpha$ is optimal for all initial risk levels $\alpha \in (0, 1]$. That is, for all $\alpha \in (0, 1]$, we have
\begin{align*}
    \cvar{\alpha}{Z^{\tilde\pi_h^\alpha}} = \max_{\pi_h} \cvar{\alpha}{Z^{\pi_h}}.
\end{align*}
\end{definition}
To simplify our argument, we will assume without loss of generality that there is always a single optimal history-dependent policy for every $\alpha$. That is $\arg\min_{\pi_h} \cvar{\alpha}{Z^{\pi_h}}$ is always a singleton, where we break ties consistently for policies with the same CVaR values. 

For a policy $\tilde{\pi}$ to achieve uniform optimality, its selected actions must align with those of the optimal history-dependent policy $\pi_{h,\alpha}^{\star}$ for every value of $\alpha$ and all possible histories. By Corollary~\ref{corollary:markovian_policies_no_gap}, each optimal policy $\pi_{h,\alpha}^{\star}$ induces a unique risk-dependent policy $\tilde\pi_\alpha$ and risk level assignment $\mathcal{Y}_\alpha$. 
To be uniformly optimal, a policy $\tilde\pi$ therefore has to ensure it simultaneously follows every $\tilde\pi_\alpha$ alongside its risk level assignment $\mathcal Y_\alpha$. These requirements can be grouped in a set of \emph{optimal-action-selection constraints}.



\begin{restatable}[Uniform optimality constraints]{proposition}{propositionConditionUniformOptimality}
\label{prop:condition_for_uniform_optimality}
A risk-dependent policy $\tilde\pi$ is uniformly optimal if and only if it simultaneously satisfies all the \emph{optimal-action-selection constraints}
\begin{align*}
    \tilde{\pi}\big(S_t, \mathcal{Y}_\alpha(H_{0:t})\big) = \pi_{h,\alpha}^{\star}(H_{0:t}),
\end{align*}
defined for all initial risk levels $\alpha \in (0, 1]$, optimal policies $\pi_{h,\alpha}^{\star}$, histories $H \in \mathcal H_T$ with $P^{\pi_{h,\alpha}^{\star}}(H) > 0$, and time steps $t=0,\dots,T-1$.
\end{restatable}


Proposition~\ref{prop:condition_for_uniform_optimality} provides a clear test for uniform optimality: one must check if the set of optimal-action-selection constraints is feasible on an MDP to know whether or not there exists a uniformly optimal policy. If, for two different initial risk levels, the respective optimal policies generate the same risk-augmented state $(S,Y)$ but require different actions, then no single deterministic policy $\tilde\pi$ can satisfy all constraints simultaneously. This conflict is the basis for the following impossibility result.



\begin{theorem}[Impossible uniform optimality counterexample]
\label{thm:impossibility}
There exists an MDP satisfying Assumptions~\ref{assumptions} for which no single risk-dependent policy $\tilde\pi: \states \times (0, 1] \to \actions$ is uniformly optimal.
\end{theorem}
\begin{proof}
We prove the result by providing an example MDP for which uniform optimality is impossible. We once again use the MDP from \citet{Hau2023Dynamic}, displayed in  Figure~\ref{fig:cvar-3actionscounterexample}, with horizon $T=2$. This MDP contains only three different history-dependent policies, which are all fully characterized by the selected action in state $s_1$. We can therefore easily compute the static CVaR of history-dependent policies to see which one is optimal at different initial risk level $\alpha$. To simplify notation, let us denote the three available policies $\pi_h^{(i)}$ to indicate which action $a_i$ they select in $s_1$. 

Figure~\ref{fig:cvar_optimal_policies} shows the CVaR$_\alpha$ of each policy based on the initial risk level $\alpha$. Crucially, Figure~\ref{fig:cvar_optimal_policies} shows that the optimal action choice is strictly dependent on the initial risk level. This observation serves as the foundational mechanism for our counterexample, illustrating how action requirements for optimality at different risk levels eventually become contradictory.
We can deduce the optimal policy:
\begin{align*}
\pi_{h,\alpha}^\star \coloneq \begin{cases}
    \pi_h^{(2)} & \text{if} \quad \alpha \in [0, 0.375)\\
    \pi_h^{(3)} & \text{if} \quad \alpha \in [0.375, 0.6875)\\
    \pi_h^{(1)} & \text{if} \quad \alpha \in [0.6875, 1],
\end{cases}
\end{align*}
meaning that the optimal risk-seeking ($\alpha$ close to 1) choice is to select action $a_1$, while the optimal risk-averse ($\alpha$ close to 0) choice is instead to pick $a_2$, with $a_3$ being the optimal choice at a moderate risk-level. 

\begin{figure}[t]
\centering
\includegraphics[width=0.85\textwidth]{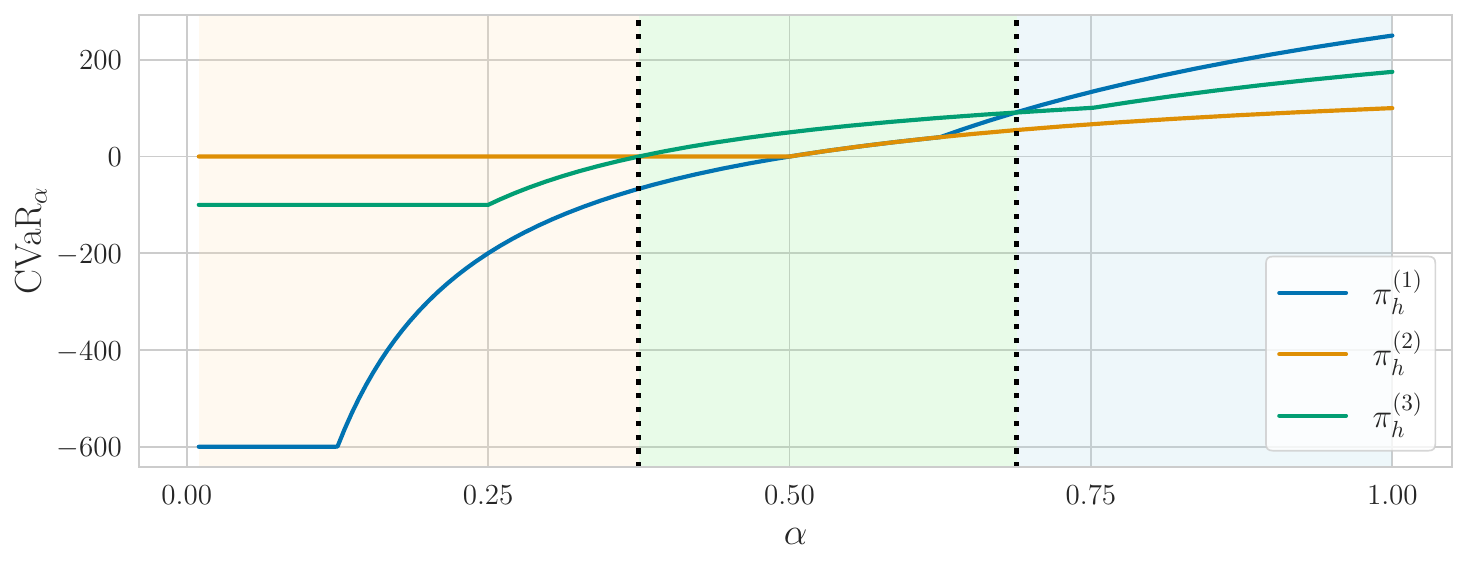}
\caption[Evolution of the static CVaR evaluation of all policies on the MDP presented in Figure~\ref{jmlr_fig:cvar-3actionscounterexample}]{Evolution of the static CVaR evaluation of all policies on the MDP presented in Figure~\ref{fig:cvar-3actionscounterexample} at different initial risk levels $\alpha$. Shaded regions represent the optimal policy $\pi_h\opt$ at a given initial risk level $\alpha$.}
\label{fig:cvar_optimal_policies}
\end{figure}

Because all three policies $\pi_h^{(i)}$ are Markovian, Corollary~\ref{corollary:markovian_policies_no_gap} applies and we know that computing their respective value functions (Eq.~\ref{eq:dp_cvar_eval}) grants us state-level perturbations $\bm{\tilde\xi}$ from which we can extract a consistent risk assignment. Combining these risk assignments with the $\alpha$ values where each $\pi_h^{(i)}$ is optimal, we can extract the optimal-action-selection constraints. To streamline our argument, let us consider the state-level risk envelope constraints imposed at state $s_1$ for cases $\alpha=0.25$, $\alpha=0.5$, and $\alpha=0.75$ in particular. For these, the optimal-action-selection-constraints (Proposition~\ref{prop:condition_for_uniform_optimality}) in $s_1$ yield:
\begin{align*}
    \pi_{h,0.25}^\star=\pi_h^{(2)} &&\implies&& \mathcal Y_{0.25}(s_1) = 0.5 &&\implies&& \tilde\pi\opt(s_1, 0.5) = a_2\\
    \pi_{h,0.5}^\star=\pi_h^{(3)} &&\implies &&\mathcal Y_{0.5}(s_1) = 0.5 &&\implies&& \tilde\pi\opt(s_1, 0.5) = a_3\\
    \pi_{h,0.75}^\star=\pi_h^{(1)} && \implies&& \mathcal Y_{0.75}(s_1) = 0.5&&\implies&& \tilde\pi\opt(s_1, 0.5) = a_1
\end{align*}
That is, in order to be simultaneously optimal for initial risk levels $\alpha=0.25$, $\alpha=0.5$, and $\alpha=0.75$, a risk-dependent policy $\tilde\pi$ would be required to take \textit{all} actions $a_1$, $a_2$, and $a_3$ in state $S_t=s_1$ when the risk level is $Y_t=0.5$, proving the impossibility of uniform optimality. 
\end{proof}


\begin{figure}[t]
\centering
\includegraphics[width=0.85\textwidth]{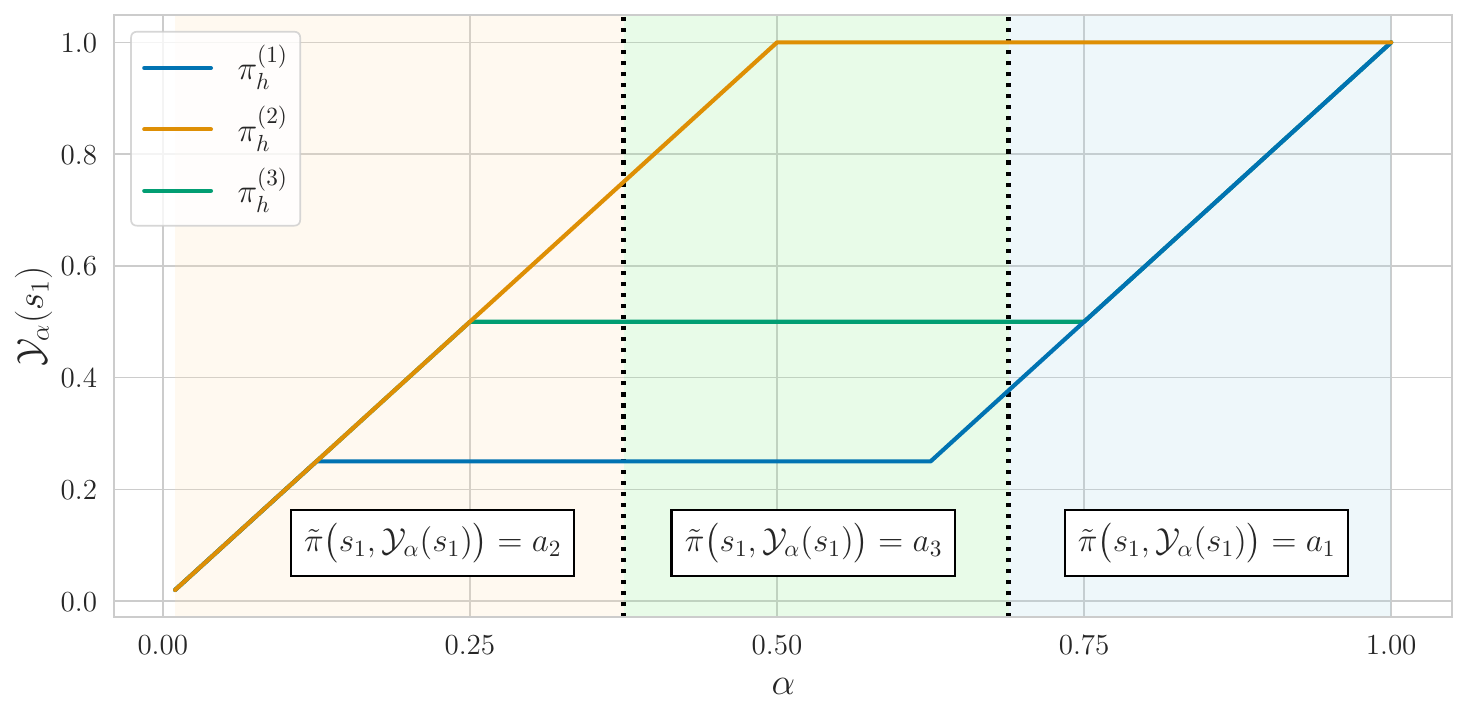}
\caption[Relation between initial risk level and the corresponding risk level for all policies on the MDP presented in Figure~\ref{jmlr_fig:cvar-3actionscounterexample}]{Relation between initial risk level $\alpha$ and the corresponding risk level $\mathcal Y_\alpha(s_1)$ for all three possible policies on the MDP presented in Figure~\ref{fig:cvar-3actionscounterexample}. Colored areas indicate the optimal policies, with the resulting constraint on $\tilde\pi$ displayed explicitly.}
\label{fig:cvar_impossibility_demo}
\end{figure}

Note that the range of values preventing the existence of a uniformly optimal policy in the MDP presented in Figure~\ref{fig:cvar-3actionscounterexample} extends beyond the $\alpha$ values presented in the proof of Theorem~\ref{thm:impossibility}. To illustrate this, Figure~\ref{fig:cvar_impossibility_demo} displays the optimal-action-selection constraints for all $\alpha \in (0, 1]$. Colored regions indicate which policy is considered optimal for the given initial risk level $\alpha$, with a box highlighting the explicit constraint for every region. For every policy, the risk-assignment in $s_1$ ($\mathcal{Y}_\alpha(s_1)$) it corresponds to for an initial risk-level $\alpha$ is shown in a separate color. The figure highlights the presence of a wide range of initial risk levels $\alpha$ where the optimal-action-selection constraints overlap, precluding the presence of a single optimal risk-dependent policy $\pi$ simultaneously optimal for these initial risk levels. 

\section{Conclusion}
\label{sec:conclusion}

In this work, we diagnosed the root cause of failures in the static dual CVaR dynamic program decomposition. We started by framing the problem of evaluating a policy's static CVaR as two distinct but related optimization tasks: one over history-level perturbations for the true static CVaR and another over state-level perturbations for the dynamic program. This perspective revealed that a CVaR evaluation gap arises precisely when a set of risk-assignment consistency constraints (Def.~\ref{def:risk_level_assignment_constraints}) have an empty intersection (Thm.~\ref{thm:absence_of_gap}). Our findings provide another independent proof that Markovian policies enjoy a null CVaR evaluation gap, which is unfortunately not the case for risk-dependent policies. Building on our risk-assignment constraints, we proved that the dual decomposition itself is fundamentally limited by identifying an MDP where no single risk-dependent policy can be uniformly optimal for all initial risk levels, as the action requirements for optimality at different risk levels become contradictory (Thm.~\ref{thm:impossibility}).

\paragraph{Future Directions}


Our findings show that seeking a single, uniformly optimal policy with the current dual decomposition approach is flawed because it can involve incompatible constraint sets. Future work should therefore pivot towards developing algorithms that find the optimal policy for a specific initial risk level, similar to methods used for primal-based decompositions~\citep{Bauerle2011Markov}. Even when limiting our attention to a single initial risk level, many interesting questions surrounding proper evaluation remain. 
\begin{enumerate}
    \item \textit{What is the extent of risk-dependent policies that have a CVaR evaluation gap at any given initial risk-level?} On one hand, we presented an instance where a risk-dependent policy presented a CVaR evaluation gap (Sec.~\ref{sec:evaluation_gap}). On the other hand, one can also easily find examples that do not suffer from a CVaR evaluation gap (App.~\ref{appendix:policy_no_evaluation_gap}). This leaves open the characterization of the occurrence of evaluation gaps. For instance, can we always find a corresponding risk-dependent policy without evaluation gap for any history-dependent policy and initial risk-level? Is this characterization dependent on intrinsic properties of the MDP or universally applicable?
    
    \item \textit{Can we determine if a risk-dependent policy will have a CVaR evaluation gap without access to its corresponding history-level perturbations?} The risk assignment consistency constraints (Def.~\ref{def:risk_level_assignment_constraints}) rely on knowledge of the corresponding history perturbations. This is problematic in practice as it requires the standalone static evaluation of a risk-dependent policy's history-dependent counterpart to evaluate whether the risk-level decomposition is accurate. One possible path forward we envision is to leverage the set of all risk-dependent policies that share the same history-dependent counterpart for a given initial risk-level. For instance, should we know that at least one of these policies does not have an evaluation gap, then as a consequence of the static CVaR upper-bound (Cor.~\ref{cor:dp_geq_static_eval}), we would know that the risk-dependent policies with lowest reported DP value would be without CVaR evaluation gap.
\end{enumerate}

Ultimately, despite the identified limitations regarding uniform optimality, the static CVaR dual decomposition retains significant practical appeal. In contrast to primal-based methods~\citep{Bauerle2011Markov}, which necessitate augmenting states with unbounded cumulative returns, the dual formulation operates within the bounded $(0,1]$ risk interval. This boundedness allows for a standardized, domain-independent discretization strategy that is more tractable in practice. Our analysis provides a foundation for further characterizing, and potentially overcoming, the challenges of risk-averse reinforcement learning. By formally characterizing the evaluation gap, our work identifies the precise constraints that any valid solution must satisfy. This analysis provides the necessary theoretical pathway for future research to derive exact, target-specific dual algorithms that reconcile these computational benefits with rigorous optimality guarantees.

\bibliography{tmlr25}

@article{Artzner1999Coherent,
  author    = {Philippe Artzner and Freddy Delbaen and Jean-Marc Eber and David Heath},
  title     = {Coherent measures of risk},
  journal   = {Mathematical Finance},
  volume    = {9},
  number    = {3},
  pages     = {203--228},
  year      = {1999}
}

@inproceedings{Baek2018Path,
  author    = {Donghoon Baek and Minho Hwang and Hansoul Kim and Dong-Soo Kwon},
  title     = {Path planning for automation of surgery robot based on probabilistic roadmap and reinforcement learning},
  booktitle = {2018 15th International Conference on Ubiquitous Robots (UR)},
  year      = {2018},
  pages     = {342-347}
}

@inproceedings{BaselIII2019,
  author    = {{Basel Committee on Banking Supervision}},
  title     = {Minimum capital requirements for market risk},
  booktitle = {Basel III: International Regulatory Framework for Banks},
  year      = {2019},
  publisher = {Bank for International Settlements}
}

@article{Bauerle2011Markov,
  author    = {Nicole B{\"a}uerle and Jonathan Ott},
  title     = {Markov decision processes with average-value-at-risk criteria},
  journal   = {Mathematical Methods of Operations Research},
  volume    = {74},
  number    = {3},
  pages     = {361--379},
  year      = {2011}
}

@article{Boda2004Stochastic,
  author    = {Kang Boda and Jerzy A. Filar and Yuanlie Lin and Lieneke Spanjers},
  title     = {Stochastic target hitting time and the problem of early retirement},
  journal   = {IEEE Transactions on Automatic Control},
  volume    = {49},
  number    = {3},
  pages     = {409--419},
  year      = {2004}
}

@inproceedings{Chapman2019Risk,
  author    = {Margaret P. Chapman and Jonathan Lacotte and Aviv Tamar and Donggun Lee and Kevin M. Smith and Victoria Cheng and Jaime F. Fisac and Susmit Jha and Marco Pavone and Claire J. Tomlin},
  title     = {A risk-sensitive finite-time reachability approach for safety of stochastic dynamic systems},
  booktitle = {American Control Conference (ACC)},
  pages     = {2958--2963},
  year      = {2019}
}

@article{Chapman2021Risk,
  author    = {Margaret P. Chapman and Riccardo Bonalli and Kevin M. Smith and Insoon Yang and Marco Pavone and Claire J. Tomlin},
  title     = {Risk-sensitive safety analysis using conditional value-at-risk},
  journal   = {IEEE Transactions on Automatic Control},
  volume    = {67},
  number    = {12},
  pages     = {6521--6536},
  year      = {2021}
}

@inproceedings{Chow2014Algorithms,
  author    = {Yinlam Chow and Mohammad Ghavamzadeh},
  title     = {Algorithms for {CVaR} optimization in {MDPs}},
  booktitle = {Advances in Neural Information Processing Systems (NeurIPS)},
  pages     = {3509--3517},
  year      = {2014}
}

@inproceedings{Chow2015Risk,
  author    = {Yinlam Chow and Aviv Tamar and Shie Mannor and Marco Pavone},
  title     = {Risk-sensitive and robust decision-making: A {CVaR} optimization approach},
  booktitle = {Advances in Neural Information Processing Systems (NeurIPS)},
  volume    = {28},
  year      = {2015}
}

@article{Chow2018Risk,
  author    = {Yinlam Chow and Mohammad Ghavamzadeh and Lucas Janson and Marco Pavone},
  title     = {Risk-constrained reinforcement learning with percentile risk criteria},
  journal   = {Journal of Machine Learning Research},
  volume    = {18},
  number    = {167},
  pages     = {1--51},
  year      = {2018}
}

@article{Ding2022Cvar,
  author    = {Rui Ding and Eugene Feinberg},
  title     = {{CVaR} optimization for {MDPs}: Existence and computation of optimal policies},
  journal   = {ACM SIGMETRICS Performance Evaluation Review},
  volume    = {50},
  number    = {2},
  pages     = {39--41},
  year      = {2022}
}

@article{Ding2022Sequential,
  author    = {Rui Ding and Eugene A. Feinberg},
  title     = {Sequential optimization of {CVaR}},
  journal   = {arXiv preprint arXiv:2211.07288},
  year      = {2022}
}

@book{Follmer2016Stochastic,
  author    = {Hans F{\"o}llmer and Alexander Schied},
  title     = {Stochastic Finance: An Introduction in Discrete Time},
  publisher = {De Gruyter},
  year      = {2016}
}

@inproceedings{Gagne2021Two,
  author    = {Christopher Gagne and Peter Dayan},
  title     = {Two steps to risk sensitivity},
  booktitle = {Advances in Neural Information Processing Systems (NeurIPS)},
  volume    = {34},
  pages     = {2361--2372},
  year      = {2021}
}

@article{Gottesman2019Guidelines,
  author    = {Omer Gottesman and Fredrik Johansson and Matthieu Komorowski and Aldo Faisal and David Sontag and Finale Doshi-Velez and Leo Anthony Celi},
  title     = {Guidelines for reinforcement learning in healthcare},
  journal   = {Nature Medicine},
  volume    = {25},
  number    = {1},
  pages     = {16--18},
  year      = {2019}
}

@inproceedings{Hau2023Dynamic,
  author    = {Jia Lin Hau and Erick Delage and Mohammad Ghavamzadeh and Marek Petrik},
  title     = {On dynamic programming decompositions of static risk measures in Markov decision processes},
  booktitle = {Advances in Neural Information Processing Systems (NeurIPS)},
  volume    = {36},
  pages     = {51734--51757},
  year      = {2023}
}

@inproceedings{Li2010Contextual,
  author    = {Lihong Li and Wei Chu and John Langford and Robert E. Schapire},
  title     = {A contextual-bandit approach to personalized news article recommendation},
  booktitle = {Proceedings of the International Conference on World Wide Web (WWW)},
  pages     = {661--670},
  year      = {2010}
}

@article{Li2022Quantile,
  author    = {Xiaocheng Li and Huaiyang Zhong and Margaret L. Brandeau},
  title     = {Quantile Markov decision processes},
  journal   = {Operations Research},
  volume    = {70},
  number    = {3},
  pages     = {1428--1447},
  year      = {2022}
}

@article{Mannor2007Bias,
  author   = {Mannor, Shie and Simester, Duncan and Sun, Peng and Tsitsiklis, John N},
  title    = {Bias and variance approximation in value function estimates},
  journal  = {Management Science},
  volume   = {53},
  number   = {2},
  pages    = {308--322},
  year     = {2007}
}

@article{Pflug2016Time,
  author    = {Georg Ch. Pflug and Alois Pichler},
  title     = {Time-consistent decisions and temporal decomposition of coherent risk functionals},
  journal   = {Mathematics of Operations Research},
  volume    = {41},
  number    = {2},
  pages     = {682--699},
  year      = {2016}
}

@article{Prashanth2022Risk,
  author    = {L. A. Prashanth and Michael C. Fu and others},
  title     = {Risk-sensitive reinforcement learning via policy gradient search},
  journal   = {Foundations and Trends in Machine Learning},
  volume    = {15},
  number    = {5},
  pages     = {537--693},
  year      = {2022}
}

@book{Puterman2014Markov,
  author    = {Martin L. Puterman},
  title     = {Markov Decision Processes: Discrete Stochastic Dynamic Programming},
  publisher = {John Wiley \& Sons},
  year      = {2014}
}

@inproceedings{Rigter2021Risk,
 author     = {Marc Rigter and Bruno Lacerda and Nick Hawes},
 title      = {Risk-averse Bayes-adaptive reinforcement learning},
 booktitle  = {Advances in Neural Information Processing Systems (NeurIPS)},
 pages      = {1142--1154},
 volume     = {34},
 year       = {2021}
}

@article{Rockafellar2000Optimization,
  author    = {R. Tyrrell Rockafellar and Stanislav Uryasev},
  title     = {Optimization of conditional value-at-risk},
  journal   = {Journal of Risk},
  volume    = {2},
  number    = {3},
  pages     = {21--42},
  year      = {2000}
}

@book{Shapiro2014Lectures,
  author    = {Alexander Shapiro and Darinka Dentcheva and Andrzej Ruszczy{\'n}ski},
  title     = {Lectures on Stochastic Programming: Modeling and Theory},
  publisher = {SIAM},
  year      = {2014}
}

@article{Silver2018General,
  author    = {David Silver and Thomas Hubert and Julian Schrittwieser and Ioannis Antonoglou and Matthew Lai and Arthur Guez and Marc Lanctot and Laurent Sifre and Dharshan Kumaran and Thore Graepel and others},
  title     = {A general reinforcement learning algorithm that masters chess, shogi, and Go through self-play},
  journal   = {Science},
  volume    = {362},
  number    = {6419},
  pages     = {1140--1144},
  year      = {2018}
}

@inproceedings{Stanko2019Risk,
  author    = {Silvestr Stanko and Karel Macek},
  title     = {Risk-averse distributional reinforcement learning: A {CVaR} optimization approach},
  booktitle = {Proceedings of the International Joint Conference on Computational Intelligence (IJCCI)},
  pages     = {412--423},
  year      = {2019}
}

@book{Sutton2018Reinforcement,
  author    = {Richard S. Sutton and Andrew G. Barto},
  title     = {Reinforcement Learning: An Introduction},
  publisher = {A Bradford Book},
  year      = {2018}
}

@book{Szepesvari2022Algorithms,
  author    = {Csaba Szepesv{\'a}ri},
  title     = {Algorithms for Reinforcement Learning},
  publisher = {Springer Nature},
  year      = {2022}
}

@inproceedings{Tamar2015Optimizing,
  author    = {Aviv Tamar and Yonatan Glassner and Shie Mannor},
  title     = {Optimizing the {CVaR} via sampling},
  booktitle = {Proceedings of the AAAI Conference on Artificial Intelligence (AAAI)},
  volume    = {29},
  year      = {2015}
}

@article{Vinyals2019Grandmaster,
  author    = {Oriol Vinyals and Igor Babuschkin and Wojciech M. Czarnecki and Micha{\"e}l Mathieu and Andrew Dudzik and Junyoung Chung and David H. Choi and Richard Powell and Timo Ewalds and Petko Georgiev and others},
  title     = {Grandmaster level in {StarCraft II} using multi-agent reinforcement learning},
  journal   = {Nature},
  volume    = {575},
  number    = {7782},
  pages     = {350--354},
  year      = {2019}
}
\bibliographystyle{tmlr}
\clearpage

\appendix

\section{Notation}
\label{app:notation}

\begin{table}[h]
    \centering
    \small
    \renewcommand{\arraystretch}{1.35} 
    \begin{tabular}{|c|l|}
        \hline
        \textbf{Symbol} & \textbf{Description} \\
        \hline
        \( \mathcal{S} \) & Finite state space \\
        \( \mathcal{A} \) & Finite action space \\
        \( \Delta(\cdot) \) & Probability simplex \\
        \( P \) & Transition function \( P: \mathcal{S} \times \mathcal{A}\mapsto \Delta(\mathcal S) \) \\
        \( \mathcal R \) & Reward function \( \mathcal R: \mathcal{S} \times \mathcal{A} \times \mathcal{S} \mapsto [0, R_{\text{max}}] \) \\
        \( \gamma \) & Discount factor \\
        \( s_0 \) & Initial state \\
         \( T \) & Trajectory horizon \\
         \( H \) & History $(S_0,A_0,S_1,\dots)$ \\
         \( \mathcal H_t \) & Set of histories of length $t$ \\
         \( H_{0:k} \) & $k$-length subsequence of history $H$ \\
         \( \pi_h \) & History-dependent policy $\pi_h:\mathcal H \mapsto \mathcal A$\\
         \( \pi \) & Markovian policy $\pi:\mathcal S \mapsto \mathcal A$\\
         \( P^{\pi_h}(H) \) & Probability of history $H$ when following policy $\pi_h$\\
         \( \mathcal R^H_{0:T} \) & Return of history $H$ \\
         \( Z^{\pi_h} \) & Random return when sampling histories by following policy $\pi_h$ \\
         \( \Xi_\alpha(\mathbb P) \) & CVaR$_\alpha$ risk envelope around probability distribution $\mathbb P$ \\
         \( \xi \) & History-level perturbations $\xi:\mathcal H \mapsto [0, \frac{1}{\alpha}]$, $\xi \in \Xi_\alpha(P^{\pi})$ \\
         \( Y \) & Running risk level $Y\in(0, 1]$ \\
         \( \tilde{\mathcal S}\) & Risk-augmented state space $\tilde{\mathcal S} \coloneq \mathcal S \times (0, 1]$ \\
         \( \tilde \pi \) & Risk-dependent policy $\tilde \pi: \tilde{\mathcal S} \mapsto \mathcal A$ \\
         \( \bm{\tilde \xi} \) & State-level perturbations $\bm{\tilde \xi}(s,a,y) \in \Xi_y(P(\cdot|s,a))$\\
         \( \bm{\tilde \Xi} \) & All valid state-level perturbations $\bm{\tilde \Xi} \coloneq \left\{ \bm{\tilde \xi} | \newline \, \bm{\tilde \xi}(\cdot | s, y, a) \in \cvarEnv{y}{P(\cdot|s,a)} \forall (s, y, a) \in \states \times (0, 1] \times \actions \right\}$ \\
         \( \tilde \pi_h^\alpha \) & History-dependent correspondance of risk-dependent policy $\tilde \pi$ at initial risk level $\alpha \in (0, 1]$\\
         \( \V{\tilde\pi,\bm{\tilde\xi}}{} \) & Policy-perturbations value function $\V{\tilde\pi,\bm{\tilde\xi}}{}: \mathcal S \times \mathcal (0, 1] \mapsto \mathbb R$\\
         \( \zeta_\alpha^{\bm{\tilde\xi}} \) & Perturbation mapping $\zeta_\alpha^{\bm{\tilde\xi}} \in \Xi_\alpha(P^{\pi_h})$ \\
         \( \mathcal Y \) & risk level assignment $\mathcal Y: \mathcal H \mapsto (0, 1]$\\
        \hline
    \end{tabular}
    \caption{List of notations}
    \label{tab:notations}
\end{table}

\clearpage

\section{Omitted proofs}

We now restate the results from the main body and present their omitted proofs. For every result, we also include a proof outline to provide a clearer picture of the proof's content.

\subsection{Results of Section~\ref{section:two_optim_problems}}

\lemmaDpOptimProblem*

\paragraph{\textit{Proof outline.}}
The result is proven by contradiction. We assume that there exist state-level perturbations yielding a value strictly lower than the minimum defined by the dynamic program at the final time step at any state-risk level pair $(s,y)$. We then demonstrate that due to the recursive nature of the value function (Def.~\ref{def:policy_pertub_value_function}), this strict inequality must propagate backwards through time, ultimately implying that the inequality must hold at time $t=0$, contradicting the uniform initialization of the value function to zero.

\begin{proof}

We prove the result by contradiction. First observe that by definition of $\V{\tilde\pi}{}$, there must exist $\bm{\xi} \in  \bm{\tilde\Xi}$ such that $\V{\tilde \pi}{} \coloneq \V{\tilde\pi, \bm{\tilde\xi}}{}$. Let $\bm{\tilde\xi}\opt\coloneq \argmin_{\bm{\tilde\xi} \in \bm{\tilde\Xi}} \V{\tilde\pi,\bm{\tilde\xi}}{}(s, y)$ and suppose there exists $\bm{\tilde\xi}' \in  \bm{\tilde\Xi}$ and $(s, y)$ such that $\V{\tilde\pi,\bm{\tilde\xi}'}{}(s, y) < \V{\tilde\pi,\bm{\tilde\xi}\opt}{}(s, y)$. Then, by the definition of $\V{\tilde\pi,\bm{\tilde\xi}\opt}{}(s, y)$ (Eq.~\ref{eq:policy_perturb_value_function}), we must have 
\begin{align*}
    \V{\tilde\pi, \bm{\tilde\xi}'}{T}(s, y) < \V{\tilde\pi, \bm{\tilde\xi}\opt}{T}(s, y).
\end{align*}

Let $a = \tilde\pi(s, y)$. Expanding the LHS using the definition of the policy-perturbations value function (Eq.~\ref{eq:policy_perturb_value_function}) and the RHS using the dynamic program definition (Eq.~\ref{eq:dp_cvar_eval}), we get
\begin{align*}
    & \, \sum_{s'} P(s'|s,a) \bm{\tilde\xi}'(s'|s,y,a) \left[\mathcal R(s, a, s') + \gamma \V{\tilde\pi, \bm{\tilde\xi}'}{T-1}(s', y') \right] \\
     < & \, \min_{\tilde\xi \in \cvarEnv{y}{P(\cdot|s,a)}} \sum_{s'} P(s'|s,a) \tilde\xi(s') \left[\mathcal R(s, a, s') + \gamma \V{\tilde\pi, \bm{\tilde\xi}\opt}{T-1}(s', y') \right].
\end{align*}

Since $\bm{\tilde\xi}'(\cdot|s,y,a) \in \cvarEnv{y}{P(\cdot|s,a)}$, the above minimum is upper bounded by substituting $\bm{\tilde\xi}'$ directly:
\begin{align*}
   & \, \min_{\tilde\xi  \in \cvarEnv{y}{P(\cdot|s,a)}} \sum_{s'} P(s'|s,a) \tilde\xi(s') \left[\mathcal R(s, a, s') + \gamma \V{\tilde\pi, \bm{\tilde\xi}\opt}{T-1}(s', y') \right] \\
    \leq & \, \sum_{s'} P(s'|s,a) \bm{\tilde\xi}'(s'|s,y,a) \left[\mathcal R(s, a, s') + \gamma \V{\tilde\pi, \bm{\tilde\xi}\opt}{T-1}(s', y') \right].
\end{align*}
Combining these inequalities and canceling the common terms, we obtain
\begin{align*}
    \sum_{s'} P(s'|s,a) \bm{\tilde\xi}'(s'|s,y,a) \V{\tilde\pi, \bm{\tilde\xi}'}{T-1}(s', y') < \sum_{s'} P(s'|s,a) \bm{\tilde\xi}'(s'|s,y,a) \V{\tilde\pi, \bm{\tilde\xi}\opt}{T-1}(s', y').
\end{align*}

For the above inequality to hold, there must exist at least one $(s', y')$ such that $\V{\tilde\pi, \bm{\tilde\xi}'}{T-1}(s', y') < \V{\tilde\pi, \bm{\tilde\xi\opt}}{T-1}(s', y')$.

Repeating this argument $T$ times implies the existence of $(s, y)$ such that $\V{\tilde\pi, \bm{\tilde\xi}'}{0}(s, y) < \V{\tilde\pi, \bm{\tilde\xi}\opt}{0}(s, y)$, which is impossible because they all are initialized to 0.
\end{proof}

\propositionConnexionDpCvar*

\paragraph{\textit{Proof outline.}}
The proof proceeds in two steps. First, we establish that the mapping $\zeta_\alpha^{\bm{\tilde\xi}}$ produces valid history perturbations (i.e., within the risk envelope) by verifying that the cumulative product of state perturbations satisfies the required probability and boundedness constraints. Second, we prove the equality between the history-level summation and the recursive value function. This is achieved via backward induction, showing that the single-step updates aggregate exactly to the return weighted by the constructed history perturbations.


\begin{proof}
We first prove $\zeta_\alpha^{\bm{\tilde\xi}} \in \cvarEnv{\alpha}{P^{\tilde\pi_h^\alpha}}$ and then we prove $\sum_{H \in \mathcal H_T} P^{\tilde\pi_h^\alpha}(H)\, \zeta_\alpha^{\bm{\tilde\xi}}(H)\, \mathcal R_{0:T}^H = \V{\tilde\pi,\bm{\tilde\xi}}{}(s_0, \alpha)$.

\paragraph{Proof of $\zeta_\alpha^{\bm{\tilde\xi}} \in \cvarEnv{\alpha}{P^{\tilde\pi_h^\alpha}}$.}
First observe that, when fixing $\bm{\tilde\xi}$ and $\alpha$, the resulting risk level at time $t$ for history $H \in \mathcal H_T$ is always given recursively by $Y_t(H|\bm{\tilde\xi}, \alpha)= \alpha\prod_{\tau=1}^{t-1}\bm{\tilde\xi}(S_{\tau+1}|S_\tau, A_\tau,Y_\tau)$ with initial condition $Y_0=\alpha$. To alleviate notation, we will mute these dependencies and simply write $Y_t$ when $H$, $\bm{\tilde \xi}$, and $\alpha$ are clear from context. Also note that we can exclude trajectories $H$ such that $P^{\tilde\pi_h^\alpha}(H) = 0$ from our analysis because these do not induce any constraint in the definition of $\cvarEnv{\alpha}{P^{\tilde\pi_h^\alpha}}$ and they are always excluded from the computation of $\V{\tilde\pi,\bm{\tilde\xi}}{}$. As a result, in the rest of the proof we can always assume that we have $A_t=\tilde\pi(S_t, Y_T)$ when iterating over a history's actions.

For a fixed $H \in \mathcal H_T$ and $\alpha$, we have by definition $\zeta_\alpha^{\bm{\tilde\xi}}(H)=\prod_{t=0}^{T-1} \bm{\tilde\xi}(S_{t+1}|S_t,Y_t,A_t)$, which is nothing more than $Y_T/\alpha$ from our above remark. Because $\bm{\tilde\xi}(S_t|S_{t-1}, Y_{t-1}, A_{t-1}) \in [0, 1/Y_{t-1}]$ for all $t=1,\dots,T$, we also have that $Y_t = Y_{t-1}\bm{\tilde\xi}(S_t|S_{t-1}, Y_{t-1}, A_{t-1}) \in (0, 1]$. Hence we have in particular $Y_T \in (0, 1]$. From the observation that $Y_T= \alpha \zeta_\alpha^{\bm{\tilde\xi}}(H)$, it follows that $\zeta_\alpha^{\bm{\tilde\xi}}(H) \in [0,1/\alpha]$ (i).



To prove $\sum_{H \in \mathcal H_T} P^{\tilde\pi_h^\alpha}(H)\, \zeta_\alpha^{\bm{\tilde\xi}}(H) = 1$, we proceed by backward induction:
\begin{align*}
& \sum_{H \in \mathcal H_T} P^{\tilde \pi_h^\alpha}(H)\, \zeta_\alpha^{\bm{\tilde\xi}}(H) = \sum_{H \in \mathcal H_T} \prod_{t=0}^{T-1} P(S_{t+1}|S_t,A_t)\, \bm{\tilde\xi}(S_{t+1}|S_t, Y_t, A_t)\\
\overset{(a)}{=} &  \sum_{H \in \mathcal H_{T-1}} \sum_{S_T \in \states} \prod_{t=0}^{T-1} P(S_{t+1}|S_t,A_t)\, \bm{\tilde\xi}(S_{t+1}|S_t, Y_t, A_t)\\
\overset{(b)}{=} & \sum_{H \in \mathcal H_{T-1}} \prod_{t=0}^{T-2} P(S_{t+1}|S_t,A_t)\, \bm{\tilde\xi}(S_{t+1}|S_t, Y_t, A_t)\sum_{S_T \in \states} P(S_T|S_{T-1},A_{T-1})\, \bm{\tilde\xi}(S_T|S_{T-1},Y_{T-1},A_{T-1})\\
\overset{(c)}{=} & \sum_{H \in \mathcal H_{T-1}} \prod_{t=0}^{T-2} P(S_{t+1}|S_t,A_t)\, \bm{\tilde\xi}(S_{t+1}|S_t, Y_t, A_t)
\end{align*}
For step (a) we used the observation that the set of histories of length $T$ is the set of histories at time $T-1$, upon which we concatenate the action $A_{T-1}$ selected by $\tilde\pi_h^\alpha$ and the possible next states $S_T$. In step (b) we exploited the fact that we can peel off the terms depending on $S_T$ from the product. Step (c) exploited the fact that $\bm{\tilde\xi}(\cdot| s, y, a) \in \cvarEnv{y}{P(\cdot | s, a}$ for all $a=\tilde\pi(s,y)$. Repeating the above manipulation $T$ times yields
\begin{align*}
\sum_{H \in \mathcal H_T} P^{\tilde \pi_h^\alpha}(H)\, \zeta_\alpha^{\bm{\tilde\xi}}(H) = \dots = \sum_{S_1 \in \states} P(S_1|S_0,A_0)\, \bm{\tilde\xi}(S_1|S_0,Y_0,A_0) = 1,
\end{align*}
, where we used the fact that $\bm{\tilde\xi}(\cdot|S_0,Y_0,A_0) \in \cvarEnv{Y_0}{P(\cdot|S_0, A_0}$ for the last equality, yielding the desired property (ii). Combining properties (i) and (ii) proves $\zeta_\alpha^{\bm{\tilde\xi}} \in \cvarEnv{\alpha}{P^{\tilde\pi_h^\alpha}}$.

\paragraph{Proof of $\sum_{H \in \mathcal H_T} P^{\tilde\pi_h^\alpha}(H)\, \zeta_\alpha^{\bm{\tilde\xi}}(H)\, \mathcal R_{0:T}^H = \V{\tilde\pi,\bm{\tilde\xi}}{}(s_0, \alpha)$.}
We now prove the stated equality by proving a slightly more general result. Let $\mathcal H_T^{s}$ denote the set of all histories $H \in \mathcal H_T$ beginning with $S_0=s$. We shall prove that the following result holds for all $s \in \states$ and $\alpha \in (0, 1]$:
\begin{align*}
\sum_{H \in \mathcal H_T^s} P^{\tilde\pi_h^\alpha}(H)\, \zeta_\alpha^{\bm{\tilde\xi}}(H)\, \mathcal R_{0:T}^H = \V{\tilde\pi,\bm{\tilde\xi}}{}(s, \alpha).
\end{align*}
We proceed by induction.
\paragraph{Base case $t=1$.}
Let $\alpha \in (0, 1]$, $s \in \states$ and set $Y_0=\alpha$, we get
\begin{align*}
\V{\tilde\pi,\bm{\tilde\xi}}{1}(s, Y_0)
=& \sum_{s'\in\states} P(s' | s, A_0)\, \bm{\tilde\xi}(s'|s,Y_0,A_0) \left[\mathcal R(s,A_0,s') + \gamma \V{\tilde\pi}{0}(s', y')\right] \\
=& \sum_{s'\in\states} P(s' | s, A_0)\, \bm{\tilde\xi}(s'|s,Y_0,A_0)\, \mathcal R(s,A_0,s') \\
=& \sum_{H \in \mathcal H_1^s} P^{\tilde\pi_h^\alpha}(H)\, \zeta_\alpha^{\bm{\tilde\xi}}(H)\, \mathcal R_{0:1}^H,
\end{align*}
where we used $A_0=\tilde\pi(s,Y_0)$ and $y'=Y_0\, \bm{\tilde\xi}(s'|s,Y_0,A_0)$ to alleviate notation. The result holds for all $s \in \states$ and $\alpha \in (0, 1]$.

\paragraph{Induction step.}
Suppose $\V{\tilde\pi,\bm{\tilde\xi}}{t-1}(s, \alpha) =  \sum_{H \in \mathcal H_{t-1}^s} P^{\tilde\pi_h^\alpha}(H)\, \zeta_\alpha^{\bm{\tilde\xi}}(H)\, \mathcal R_{0:t-1}^H$. Again setting $Y_0=\alpha$, we get
\begin{align*}
\V{\tilde\pi,\bm{\tilde\xi}}{t}(s, Y_0)
=& \sum_{s'\in\states} P(s' | s, A_0)\, \bm{\tilde\xi}(s'|s,Y_0,A_0) \left[\mathcal R(s,A_0,s') + \gamma \V{\tilde\pi}{t-1}(s', y')\right] \\
=& \sum_{s'\in\states} P(s' | s, A_0)\, \bm{\tilde\xi}(s'|s,Y_0,A_0)\, \mathcal R(s,A_0,s') + \gamma \sum_{H' \in \mathcal H_{t-1}^{s'}} P^{\tilde\pi_h^{y'}}(H')\, \zeta_{y'}(\bm{\tilde\xi})(H')\, \mathcal R_{0:t-1}^{H'}\\
\overset{(a)}{=} & \sum_{s'\in\states} P(s' | s, A_0)\, \bm{\tilde\xi}(s'|s,Y_0,A_0)\left[ \sum_{H' \in \mathcal H_{t-1}^{s'}} P^{\tilde\pi_h^{y'}}(H')\, \zeta_{y'}(\bm{\tilde\xi})(H') \left(\mathcal R(s,A_0,s') + \gamma \mathcal R_{0:t-1}^{H'}\right)\right]\\
\overset{(b)}{=} & \sum_{s'\in\states} \sum_{H' \in \mathcal H_{t-1}^{s'}} P(s' | s, A_0)\, \bm{\tilde\xi}(s'|s,Y_0,A_0)\, P^{\tilde\pi_h^{y'}}(H')\, \zeta_{y'}(\bm{\tilde\xi})(H')  \left(\mathcal R(s,A_0,s') + \gamma \mathcal R_{0:t-1}^{H'}\right)\\
\overset{(c)}{=} & \sum_{H \in \mathcal H_t^s} P^{\tilde\pi_h^\alpha}(H)\, \zeta_\alpha^{\bm{\tilde\xi}}(H)\, \mathcal R_{0:t}^{H},
\end{align*}
where we used $A_0=\tilde\pi(s,Y_0)$ and $y'=y_0\,\bm{\tilde\xi}(s'|s,Y_0,A_0)$ to alleviate notation. In step (a), we exploited the fact that $\sum_{H' \in \mathcal H_{t-1}^{s'}} P^{\tilde\pi_h^{y'}}(H')\, \zeta_{y'}(\bm{\tilde\xi})(H')=1$ to put the $\mathcal R(s, A_0, s')$ term inside the sum. Step (b) used the fact that the $P(s' | s, A_0)$ and $\bm{\tilde\xi}(s'|s,Y_0,A_0)$ terms do not rely on $H'$ to move the summation. Finally, step (c) relied on the fact that all trajectories starting from $S_0=s$ can be expressed as the union over $s'$ of all trajectories trajectories with $S_1=s'$ and updating the necessary probability, perturbations, and reward terms accordingly. 
The result follows by setting $S=s_0$ and picking the desired $\alpha \in (0, 1]$.
\end{proof}

\subsection{Results of Section~\ref{sec:evaluation_gap}}

\lemmaRealizabilityIffConsistency*

\paragraph{\textit{Proof outline.}}
The proof is constructive for both directions. For the forward direction ($\implies$), given realizable perturbations, we construct the risk assignment $\mathcal{Y}$ recursively from the existing state-level perturbations and verify that this construction inherently satisfies the consistency constraints. For the backward direction ($\impliedby$), we construct the state-level perturbations by taking the ratio of the consistent risk assignments at subsequent steps, showing that the resulting perturbations are valid members of the risk envelope and aggregate to the target history perturbation.

\begin{proof}

($\bm \implies$) Assume $\xi$ is realizable by some $\bm{\tilde\xi}$. Define the assignment $\mathcal{Y}(H_t)$ by recursively computing the risk levels generated when following $\bm{\tilde\xi}$, that is $\mathcal{Y}(H_{0:t+1}) = \mathcal{Y}(H_{0:t}) \cdot \bm{\tilde\xi}(S_{t+1}|S_t, \mathcal{Y}(H_{0:t}), A_t)$ for all $t =0,\dots, T-1$ and initialized at $\mathcal{Y}(H_{0:0})=\alpha$. Because $\xi$ is realizable, it follows that $\mathcal{Y}(H)=\zeta_\alpha^{\bm{\tilde\xi}}(H)=\xi(H)$ for all histories $H \in \mathcal H_T$, establishing the risk propagation constraint. Now observe that by construction we have:
\begin{align*}
    \sum_{s' \in \states} P(s'|S_t, A_t) \frac{\mathcal{Y}(H_{0:t} \cup (A_t,s'))}{\mathcal{Y}(H_{0:t})} = \sum_{s' \in \states} P(s'|S_t, A_t)\, \bm{\tilde\xi}(s'|S_t, \mathcal Y(H_{0:t}), A_t).
\end{align*}
By this observation and the definition of $\bm{\tilde\xi} \in \bm{\tilde\Xi}$, the state-level risk envelope and action selection constraints follow, proving the desired statement.

($\bm \impliedby$) 
Assume a consistent risk level assignment $\mathcal{Y}$ exists. We construct $\bm{\tilde\xi}$ as follows: for any $(s,y,a)$ reached by a history $H$ with $P^{\tilde\pi_h^\alpha}(H)>0$, that is $s=S_t, y=\mathcal{Y}(H_{0:t}), a=A_t$, define $\bm{\tilde\xi}(s'|s,y,a) \coloneq \mathcal{Y}(H_t \cup (a,s'))/y$. For all other $(s,y,a)$, define $\bm{\tilde\xi}(s'|s,y,a)=1$ to trivially ensure it is part of $\cvarEnv{y}{P(\cdot|s,a)}$. By the state-level risk envelope and action selection constraints, it follows that $\bm{\tilde\xi} \in \bm{\tilde\Xi}$. Moreover, for all histories we have $\zeta_\alpha^{\bm{\tilde\xi}}(H) = \prod_{t=0}^{T-1} \frac{\mathcal{Y}(H_{0:t+1})}{\mathcal{Y}(H_{0:t})} = \frac{\mathcal{Y}(H_{0:T})}{\mathcal{Y}(H_{0:0})} = \frac{\alpha \cdot \xi(H)}{\alpha} = \xi(H)$. Thus, $\xi$ is realizable.
\end{proof}

\thmAbsenceOfGap*

\paragraph{\textit{Proof outline.}}
This proof relies on the equivalence established in Lemma~\ref{lemma:realizability_iff_consistency} and the upper bound property from Corollary~\ref{cor:dp_geq_static_eval}. For the forward direction, if the evaluation gap is zero, the optimal history perturbations must be realizable by some state-level perturbations, implying consistency. For the reverse direction, if consistency holds, the optimal history perturbations are realizable. This allows us to exactly match the dynamic program's value between the true static CVaR and the realizable value, proving equality.

\begin{proof}
($\bm \implies$) 
Assume $\cvar{\alpha}{Z^{\tilde\pi_h^\alpha}} = \V{\tilde\pi}{}(s_0, \alpha)$. Let $\bm{\tilde\xi}^\star$ be a minimizer for $\V{\tilde\pi,\bm{\tilde\xi}}{}(s_0, \alpha)$. From Proposition~\ref{proposition:connexion_dp_cvar} we know we can compute $\xi' \coloneq \zeta_\alpha(\bm{\tilde\xi}^\star)$, where $\xi' \in \cvarEnv{\alpha}{P^{\tilde\pi_h^\alpha}}$ are valid trajectory perturbations and $\sum_{H \in \mathcal H_T} P^{\tilde\pi_h^\alpha}(H)\, \xi'(H)\, \mathcal R_{0:T}^H = \V{\tilde\pi,\bm{\tilde\xi}^\star}{}(s_0, \alpha) = \V{\tilde\pi}{}(s_0, \alpha)$. Since we assumed $\cvar{\alpha}{Z^{\tilde\pi_h^\alpha}} = \V{\tilde\pi}{}(s_0, \alpha)$, this means $\xi'$ must be an optimal trajectory perturbations set. Because $\xi'$ is realizable by construction, applying Lemma~\ref{lemma:realizability_iff_consistency} gives the desired result.

($\bm \impliedby$) 
Let $\xi\opt$ be an optimal trajectory perturbations set for which there exists a consistent risk assignment. By Lemma~\ref{lemma:realizability_iff_consistency}, we know $\xi\opt$ is realizable and hence there exists a $\bm{\tilde\xi}' \in \bm{\tilde\Xi}$ such that $\zeta_\alpha(\bm{\tilde\xi}')=\xi\opt$. From Proposition~\ref{proposition:connexion_dp_cvar} and the definition of $\xi\opt$, we have $\V{\tilde\pi,\bm{\tilde\xi}'}{}(s_0,\alpha) = \sum_{H \in \mathcal H_T} P^{\tilde\pi_h^\alpha}(H)\, \xi\opt(H)\, \mathcal R_{0:T}^H = \cvar{\alpha}{Z^{\tilde\pi_h^\alpha}}$. Since $\V{\tilde\pi}{}(s_0, \alpha) = \min_{\bm{\tilde\xi}} \V{\tilde\pi, \bm{\tilde\xi}}{}(s_0, \alpha)$, it must be that $\V{\tilde\pi}{}(s_0, \alpha) \le \V{\tilde\pi,\bm{\tilde\xi}\opt}{}(s_0,\alpha) = \cvar{\alpha}{Z^{\tilde\pi_h^\alpha}}$. Combining this with the opposite inequality from Corollary~\ref{cor:dp_geq_static_eval} yields equality.
\end{proof}

\corHistPolNoGap*

\paragraph{\textit{Proof outline.}}
The result is proven by construction. We show that a Markovian policy can be trivially viewed as a risk-dependent policy where the action choice is independent of the risk level $y$. This renders the action-selection consistency constraints non-binding. Consequently, the standard CVaR decomposition theorem guarantees the existence of a valid risk assignment, ensuring no gap exists.

\begin{proof}

We prove the statement by explicitly proving that $\tilde \pi = \pi$ is the desired policy. First observe that one can obtain a risk assignment $\mathcal Y$ by leveraging the CVaR Decomposition Theorem (Thm.~\ref{thm:cvar_decomposition} $T$ times, for instance by applying the dynamic programming operator in Equation~\ref{def:dp_cvar_eval}. By construction, this risk assignment will satisfy the risk propagation and state-level risk envelope constraints. Considering $\tilde{\pi}=\pi$, the action-selection consistency constraints are also trivially satisfied, ensuring $\mathcal Y$ is a consistent risk assignment mapping from $\tilde \pi$ and $\pi$, hence $\tilde \pi$ does not have a CVaR evaluation gap as per Theorem~\ref{thm:absence_of_gap}.

\end{proof}

\subsection{Results of Section~\ref{sec:impossibility_result}}

\propositionConditionUniformOptimality*

\paragraph{\textit{Proof outline.}}
We prove the equivalence by linking the definitions of uniform optimality and action consistency. In the forward direction, uniform optimality implies the policy must match the optimal history-dependent policy for every $\alpha$, which in turn implies matching the actions prescribed by those optimal policies at the reachable risk-augmented states. In the reverse direction, satisfying the constraints ensures the policy mimics the optimal history-dependent policy for every $\alpha$, thereby guaranteeing the CVaR values match.

\begin{proof}
($\bm{\implies}$) Assume $\tilde\pi$ is uniformly optimal. By definition, for any $\alpha \in (0, 1]$, we have $\cvar{\alpha}{Z^{\tilde\pi_h^\alpha}} = \max_{\pi_h}\cvar{\alpha}{Z^{\pi_h}} = \cvar{\alpha}{Z^{\pi_{h,\alpha}^{\star}}}$. Because we assumed there is only one optimal $\pi_h\opt(\alpha)$, it follows that the history distributions $P^{\tilde \pi_h^\alpha}$ and $P^{\pi_h\opt(\alpha)}$ must match. Because the history distributions only depend on policies in whether the action $A_t$ is selected, it follows that $\pi_h^\alpha$ and $\pi_h\opt(\alpha)$ always select the same action $A_t$ given $H_{0:t}$. Because of the feasibility of the dynamic program decomposition for $\pi_h(\alpha)$ (Corollary~\ref{corollary:markovian_policies_no_gap}), this is equivalent to saying $\tilde{\pi}\big(S_t, \mathcal{Y}_\alpha(H_{0:t})\big) = \pi_{h,\alpha}^{\star}(H_{0:t})$. As this must hold for all $\alpha \in (0, 1]$, $\tilde\pi$ must satisfy the optimal-action-selection constraints.

($\bm{\impliedby}$) Assume $\tilde\pi$ satisfies the optimal-action-selection constraints. This means that for any given $\alpha \in (0, 1]$ and any history $H_{0:t}$ reachable under the optimal policy $\pi_{h,\alpha}^{\star}$, the action chosen by $\tilde\pi_h^\alpha$ is the same as the action chosen by $\pi_{h,\alpha}^{\star}$. This ensures that the two policies are identical, $\tilde\pi_h^\alpha = \pi_{h,\alpha}^{\star}$. Consequently, their induced trajectory distributions are identical, $P^{\tilde\pi_h^\alpha} = P^{\pi_{h,\alpha}^{\star}}$. It follows directly that their CVaR evaluations must also be identical:
\begin{align*}
    \cvar{\alpha}{Z^{\tilde\pi_h^\alpha}} = \cvar{\alpha}{Z^{\pi_{h,\alpha}^{\star}}} = \max_{\pi_h}\cvar{\alpha}{Z^{\pi_h}}.
\end{align*}
Since this holds for all $\alpha \in (0, 1]$, the policy $\tilde\pi$ is uniformly optimal.
\end{proof}

\newpage

\section{Example Risk-dependent Policy without CVaR Evaluation Gap}
\label{appendix:policy_no_evaluation_gap}

After having dedicated a large portion of Section~\ref{sec:evaluation_gap} to present an example where a risk-dependent policy (specifically the one returned by CVaR VI~\cite{Chow2015Risk}) has a CVaR evaluation gap, we now take the time to present a case where no such evaluation gap exists. The purpose of this exercise is to put the novel negative results presented in this article (Thm.~\ref{thm:impossibility}) and previous work \cite{Hau2023Dynamic} in context, highlighting that the evaluation issues do not plague \textit{all} risk-dependent policies.

From Theorem~\ref{thm:absence_of_gap}, to prove the absence of a CVaR evaluation gap it suffices to identify a risk-dependent policy whose risk assignment consistency constraints (Def.~\ref{def:risk_level_assignment_constraints}) have non-empty intersections. We once again use the sample MDP from \citet{Hau2023Dynamic} presented in Figure~\ref{fig:cvar-3actionscounterexample} with $\alpha=0.5$ as the initial risk-level. The policy we will use for our example is the following:
\begin{align*}
  \tilde\pi(s_1, y) = \begin{cases}
    a_3 & \text{if } y \ge 0.5, \\
    a_2 & \text{if } y < 0.5.
  \end{cases} 
\end{align*}
Note that this policy differs from the one returned by CVaR VI (Sec.~\ref{sec:evaluation_gap}) only in how the boundary $y=0.5$ is handled and in that action $a_3$ replaces $a_1$ for low risk levels ($y$  near 1). Because there is always only one path to land in a given state, we will alleviate notation when possible and only use the destination state to define state-level perturbations, for instance using $\bm{\tilde\xi}\opt(s_6|y)$ instead of $\bm{\tilde\xi}\opt(s_6|s_1,a_3,y)$.

\paragraph{Computing the value function.} In order to compute the optimal state-level perturbations $\bm{\tilde\xi}\opt$ for this policy, we proceed backwards from the terminal states. First observe that the perturbations in $s_1$ when taking action $a_3$ can be computed independently of the rest of the MDP. The minimization problem at $(s_1, y, a_3)$ from the value function definition (Eq.~\ref{eq:dp_cvar_eval}) reads
\begin{align*}
  \V{\tilde\pi, a_3}{}(s_1,y) = \min_{\tilde\xi \in \cvarEnv{y}{P(\cdot|s_1,a_3)}} \Big[ 0.5\,\tilde\xi(s_6)(-100) + 0.5\,\tilde\xi(s_7)(400) \Big].
\end{align*}
Using the risk envelope constraints $0.5\,\tilde\xi(s_6)+0.5\,\tilde\xi(s_7)=1$ and $0.5\,\tilde\xi(s_6)\le 1$, we get the optimal perturbations
\begin{align*}
  \bm{\tilde\xi}\opt(s_6|y) = \frac{1}{\max(y, 0.5)} &  & \text{and} & & \bm{\tilde\xi}\opt(s_7|y) = 2 - \frac{1}{\max(y, 0.5)}.
\end{align*}
Since there is only a single state transition possible when taking action $a_2$, we also have $\bm{\tilde\xi}\opt(s_5|y)=1$. Substituting the optimal perturbations back, the resulting action-conditional value functions at $s_1$ are $\V{\tilde\pi, a_2}{}(s_1,y)=0$ and
\begin{align*}
  \V{\tilde\pi, a_3}{}(s_1,y) = \frac{0.5}{\max(y, 0.5)} \cdot (-100) + 0.5\left(2 - \frac{1}{\max(y, 0.5)}\right) \cdot 400 = 400 - \frac{250}{\max(y, 0.5)}.
\end{align*}
We also have $\V{\tilde\pi}{}(s_2,y)=200$ for all $y$, since the only available action in $s_2$ leads deterministically to $s_8$ with reward $200$.

We can now determine the optimal perturbations at $s_0$. Letting $u=\bm{\tilde\xi}(s_1|\alpha)$ with $\bm{\tilde\xi}(s_2|\alpha)=2-u$, the risk level at $s_1$ becomes $\mathcal Y(s_1)=0.5u$. We consider two cases:
\begin{itemize}
    \item \textbf{Case $u \ge 1$:} This gives $\mathcal Y(s_1)\ge 0.5$ and hence the selected action is $a_3$. The value function evaluation $\V{\tilde\pi}{}(s_0,0.5)$ is now equivalent to 
    \begin{align*}
        f(u) &= (0.5u)\V{\tilde\pi, a_3}{}(s_1,0.5u) + 0.5(2-u) \V{\tilde\pi}{}(s_2,0.5(2-u)) \\
           &= 0.5u\left(400-\frac{250}{0.5u}\right) + 0.5(2-u)\cdot 200 \\
           &= 100u - 50,
    \end{align*}
    which is increasing in $u$ and minimized at $u=1$, giving $f(1)=50$.
    \item \textbf{Case $u < 1$:} This gives $\mathcal Y(s_1) < 0.5$ and hence the selected action is $a_2$. The value function evaluation $\V{\tilde\pi}{}(s_0,0.5)$ is now equivalent to 
    \begin{align*}
      f(u) &= (0.5u)\V{\tilde\pi, a_2}{}(s_1,0.5u) + 0.5(2-u) \V{\tilde\pi}{}(s_2,0.5(2-u)) \\
           &= 0.5u \cdot 0 + 0.5(2-u)\cdot 200 \\
           &= 200 - 100u,
    \end{align*}
    which is decreasing in $u$ and achieves its infimum as $u \to 1^-$, giving $f \to 100$.
\end{itemize}
The global minimum is therefore achieved at $u=1$, yielding $\bm{\tilde\xi}\opt(s_1|\alpha) = \bm{\tilde\xi}\opt(s_2|\alpha) = 1$ and $\V{\tilde\pi}{}(s_0,0.5)=50$. 

\paragraph{Computing the true static CVaR.} The risk level remains $Y_1=0.5$ after the transition from $s_0$, so $\tilde\pi(s_1,0.5)=a_3$. The corresponding history-dependent policy therefore takes action $a_3$ when reaching $s_1$. This produces three histories with non-zero probability:
\begin{align*}
    P^{\tilde\pi_h^{0.5}}(H) = \begin{cases} 0.25 & \text{if } H=(s_0,a_1,s_1,a_3,s_6) \\ 0.25 & \text{if } H=(s_0,a_1,s_1,a_3,s_7) \\ 0.5 & \text{if } H=(s_0,a_1,s_2,a_1,s_8) \end{cases} &&
    \mathcal R^{H} = \begin{cases} -100 & \text{if } H=(s_0,a_1,s_1,a_3,s_6)\\ 400 & \text{if } H=(s_0,a_1,s_1,a_3,s_7)\\ 200 & \text{if } H=(s_0,a_1,s_2,a_1,s_8)\end{cases}
\end{align*}
Solving the static CVaR evaluation (Eq.~\ref{eq:static_cvar_policy_eval}) amounts to minimizing $-100(0.25\xi_1) + 400(0.25\xi_2) + 200(0.5\xi_3)$ subject to $0.25\xi_1+0.25\xi_2+0.5\xi_3=1$ and $\xi_i \in [0,2]$. From straightforward algebraic manipulation, we obtain the optimal solution 
\begin{align*}
    \xi_1^\star=2, && \xi_3^\star=1, && \xi_2^\star=0.
\end{align*}

\paragraph{Verifying the risk-assignment constraints.} We now verify the risk-assignment consistency constraints (Def.~\ref{def:risk_level_assignment_constraints}). Because states never repeat in this MDP, we use the destination state as shorthand for histories.
\begin{enumerate}
    \item \textbf{Risk propagation}: Directly applying trajectory perturbations $\xi\opt$, we get
    \begin{align*}
        \mathcal{Y}(s_0) = 0.5, && \mathcal{Y}(s_6) = 0.5 \cdot 2 = 1, && \mathcal{Y}(s_7) = 0.5 \cdot 0 = 0, && \text{and} &&  \mathcal{Y}(s_8) = 0.5 \cdot 1 = 0.5.
    \end{align*}
    \item \textbf{State-level risk envelope}: For $t=0$, the constraint gives $\mathcal{Y}(s_1) +\mathcal{Y}(s_2) = 1$. For $t=1$, we have
    \begin{itemize}
        \item \textbf{From $s_1$:} we get $\mathcal{Y}(s_6)+\mathcal{Y}(s_7)=2\mathcal{Y}(s_1)$, so $1+0=2\mathcal{Y}(s_1)$, yielding \begin{align*}\mathcal{Y}(s_1)=0.5.\end{align*}
        \item \textbf{From $s_2$:} we get \begin{align*}
            \mathcal{Y}(s_2) = \mathcal{Y}(s_8) =  0.5.
        \end{align*}
    \end{itemize}
    \item \textbf{Action-selection consistency}: The histories through $s_1$ require action $a_3$. According to the policy $\tilde\pi(s_1,y)$, action $a_3$ is selected when $y \ge 0.5$ hence 
    \begin{align*}
        \mathcal{Y}(s_1) \ge 0.5.
    \end{align*}
\end{enumerate}

\def\riskpropagationcolor{brown}
\def\rasenvcolor{violet}
\def\actionselectioncolor{teal}
\begin{figure}[t!]
  \centering
  \begin{tikzpicture}[->,>=stealth,shorten >=1pt,node distance=1.8cm,semithick,level distance=23mm]
    \tikzstyle{level 1}=[sibling distance=20mm]
    \tikzstyle{level 2}=[sibling distance=14mm]
    \tikzstyle{level 3}=[sibling distance=13mm]

    \node (s0) [anchor=west, label={left:{
    \begin{tabular}{c}
         $\Big\{\textcolor{\riskpropagationcolor}{\mathcal Y(s_0) = 0.5}\Big\}$ 
    \end{tabular}}}] {$s_0, a_1$} [grow'=right,->]
    child {
      node (s1) [anchor=west, label={above:{%
        \begin{tabular}{c}
          $\Big\{\textcolor{\actionselectioncolor}{\mathcal Y(s_1) \ge 0.5}\Big\} \cap \Big\{\textcolor{\rasenvcolor}{\mathcal Y(s_1) = 0.5}\Big\}$
        \end{tabular}%
      }}] {$s_1$} 
      child {
        node (s13) {$s_1, a_3$}
        child {
          node (s6) [anchor=west, label={right:{
            \begin{tabular}{c}
                 $\Big\{\textcolor{\riskpropagationcolor}{\mathcal Y(s_6) = 1}\Big\}$ 
            \end{tabular}}}] {$s_6$}
        }
        child {
          node (s7) [anchor=west, label={right:{
            \begin{tabular}{c}
                 $\Big\{\textcolor{\riskpropagationcolor}{\mathcal Y(s_7) = 0}\Big\}$ 
            \end{tabular}}}] {$s_7$}
        }
      }
    }
    child { node (s2) [anchor=west, label={below:{
        \begin{tabular}{c}
             $\Big\{\textcolor{\rasenvcolor}{\mathcal Y(s_2) = 0.5}\Big\}$ 
        \end{tabular}}}] {$s_2$}
      child {node (s21) {$s_2,a_1$}
        child {
          node (s8) [anchor=west, label={right:{
        \begin{tabular}{c}
             $\Big\{\textcolor{\riskpropagationcolor}{\mathcal Y(s_8) = 0.5}\Big\}$ 
        \end{tabular}}}] {$s_8$}
        }
      }
    };
\end{tikzpicture}
    \caption[Visual representation of the risk-assignment constraints for a policy without CVaR evaluation gap]{Visual representation of the risk-assignment constraints on the MDP from Figure~\ref{fig:cvar-3actionscounterexample}, with the policy $\tilde\pi(s_1,y)=a_3$ if $y \ge 0.5$ and $a_2$ otherwise, at $\alpha=0.5$. Risk propagation constraints are in \textcolor{\riskpropagationcolor}{brown}, state-level risk envelope constraints are in \textcolor{\rasenvcolor}{purple}, and action selection constraints are in \textcolor{\actionselectioncolor}{teal}.}
    \label{fig:risk_assignment_possibility}
\end{figure}

All three constraint sets are simultaneously satisfiable, as shown in Figure~\ref{fig:risk_assignment_possibility}. By Theorem~\ref{thm:absence_of_gap}, this confirms the absence of a CVaR evaluation gap. 

The absence of CVaR evaluation gap can also be confirmed by computing the actual static CVaR for $\tilde\pi_h^{0.5}$. Using our previously computed optimal history-level perturbations, we get

\begin{align*}
    \cvar{0.5}{Z^{\tilde\pi_h^{0.5}}} = 400\cdot0\cdot0.25+200\cdot1\cdot0.5-100\cdot2\cdot0.25 = 50 = \V{\tilde\pi}{}(s_0,0.5).
\end{align*}
Hence, there is no CVaR evaluation gap for this policy.

Notice that the history-dependent counterpart of $\tilde \pi$ is equivalent to $\pi_h^{(3)}$, the policy that takes action $a_3$ in $s_1$. From Figure~\ref{fig:cvar_optimal_policies}, we know this is actually the optimal policy in $\alpha=0.5$, making our example of a successful DP evaluation also a demonstration that it is possible for a risk-dependent policy to be both a representation of the optimal history-dependent policy and not suffer from a CVaR evaluation gap.

\end{document}